%% file: archive_ver1.tex
\documentclass{article}

\usepackage[main, final]{neurips_2026}

\usepackage[utf8]{inputenc} 
\usepackage[T1]{fontenc}    
\usepackage{hyperref}       
\usepackage{url}            
\usepackage{booktabs}       
\usepackage{amsfonts}       
\usepackage{nicefrac}       
\usepackage{microtype}      
\usepackage{xcolor}         

\usepackage{algorithm}
\usepackage{algorithmic}

\usepackage{setspace}
\usepackage{pifont}
\usepackage{amsmath}
\usepackage{amssymb}
\usepackage{mathtools}
\usepackage{amsthm}
\usepackage{thmtools}
\usepackage{thm-restate}

\usepackage{graphicx}
\usepackage[table]{xcolor}
\usepackage{tikz}

\usepackage{multirow}
\usepackage{adjustbox}
\usepackage{wrapfig}
\usepackage{caption} 
\usepackage{subcaption}

\makeatletter
\newcommand\footnoteref[1]{\protected@xdef\@thefnmark{\ref{#1}}\@footnotemark}
\makeatother

\usepackage[capitalize]{cleveref}
 \AtBeginDocument{%
    \crefname{figure}{Figure}{Figures}%
}
\newcommand{\eqlabelleft}{(}
\newcommand{\eqlabelright}{)}

\creflabelformat{equation}{\eqlabelleft#2#1#3\eqlabelright}

\theoremstyle{plain}

\theoremstyle{definition}

\theoremstyle{remark}

\input{math_commands.tex}

\title{Best-of-$N$ Guidance for Test-time Diffusion Alignment}

\author{
Richard Lee Kim\textsuperscript{\textmd{1}} \quad Yeongmin Kim\textsuperscript{\textmd{1}} \quad Gyuwon Sim\textsuperscript{\textmd{1}} \quad Taekyu Kim\textsuperscript{\textmd{1}} \\
\bfseries Minsang Park\textsuperscript{\textmd{1}} \quad Il-Chul Moon\textsuperscript{\textmd{1,2}} \\[2pt]
\textsuperscript{\textmd{1}}KAIST, \textsuperscript{\textmd{2}}summary.ai\\
\texttt{\{swimskykim,alsdudrla10,gkwlaks4886,taekyu,pagemu,icmoon\}@kaist.ac.kr}\\
}
\begin{document}

\maketitle

\begin{abstract}
Diffusion models achieve strong generative performance but often struggle to align generated samples with human preferences measured by a reward model. A simple yet effective algorithm for test-time alignment is Best-of-$N$ (BoN) sampling, which draws $N$ i.i.d. samples from a pre-trained diffusion model and outputs the single highest-reward sample. Despite its empirical success, BoN makes limited use of reward information, as it is incorporated only at the final selection stage without influencing the reverse diffusion trajectory during sampling. Consequently, BoN sampling does not improve the average alignment of generated samples and is primarily suited to single-output settings. We propose Best-of-$N$ Guidance (BoNG), a novel method that integrates the principle of BoN sampling directly into the reverse diffusion process. BoNG performs online BoN selection over denoising particles and adjusts the reverse diffusion process to steer the particle population toward higher-reward regions during generation. Specifically, by introducing an asymmetric guidance interaction among denoising particles, BoNG uses the current BoN particle as a guidance signal to the rest of the particle population. This particle-level interaction reshapes the sampling process toward higher-reward regions, enabling BoNG to improve not only the final best sample beyond Vanilla BoN sampling, but also the average quality of generated samples. Over 36 empirical comparisons, BoNG achieves the best performance in 29 cases, ranking first in 80.56\% of the comparisons against SMC and Vanilla BoN sampling. BoNG also supports multi-output capability, achieving 1.3$\times$ ImageReward score of the latest sample-based guidance method with a 1.6$\times$ speedup. We release the code at \url{https://github.com/aailab-kaist/BoNG}.

\end{abstract}

\section{Introduction}
Diffusion models~\cite{sohl2015deep,ho2020denoising,song2021scorebased} demonstrate strong generative capabilities, leading to a growing number of real-world applications that produce outputs conditioned on user inputs~\cite{microsoft_bing,grok_imagine}. However, the generated outputs still do not fully align with user preferences in terms of input alignment and sample fidelity. To address this issue, reward functions~\cite{xu2023imagereward, wu2023human, ma2025hpsv3widespectrumhumanpreference} that evaluate alignment with user conditions and output quality have been incorporated into the generation process. Recently, test-time alignment methods have shown superior performance compared to parameter-scaling~\cite{podell2024sdxl,flux2024} or fine-tuning approaches~\cite{black2024training,wallace2024diffusion}, attracting increasing research interest in this direction~\cite{snell2024scaling,singhal2025a,kim2026lookahead}.

Best-of-N (BoN) sampling and Sequential Monte Carlo (SMC)~\cite{singhal2025a,li2025derivativefree,skreta2025feynmankac} are widely used test-time alignment methods for diffusion models. However, both approaches have inherent limitations: BoN selects only a single sample, while SMC is prone to particle collapse. As a result, they are primarily suited to scenarios where only one sample is ultimately served. In contrast, guidance-based methods enable the generation of diverse samples by independently adjusting the denoising direction for each sample, making them applicable to settings that require multiple outputs. However, early approaches, such as gradient guidance~\cite{kim2023refining,bansal2024universal,na2025diffusion}, require neural backward passes with differentiable rewards, leading to substantial computational overhead and are prone to reward hacking.

Sample-based guidance is emerging as a promising direction, as it can generate diverse samples without requiring neural backward passes. However, it often becomes a computational bottleneck, since iterative rollout sampling requires generating final samples at each time step. Consequently, this line of research has focused on reducing the cost of rollout~\cite{potaptchik2025tilt,holderrieth2026glass,potaptchik2026meta,kim2026lookahead}, and has recently achieved promising results in both performance and efficiency~\cite{kim2026lookahead}. However, it still incurs rollout cost, resulting in 2–3× higher sampling cost compared to vanilla sampling.

We propose \textit{Best-of-$N$ Guidance} (BoNG), a rollout-free, sample-based guidance method inspired by BoN sampling. Unlike Vanilla BoN which uses the reward only once to select the best sample after generation, BoNG performs this selection online and injects it into the reverse diffusion process as a guidance signal (\cref{Fig1}). Specifically, we replace the iterative rollout process by selecting the highest-reward expected denoised sample among multiple particles at several intermediate timesteps. This approach steers all samples toward the direction of the highest-reward sample without incurring additional feed-forward computations of the score network. We theoretically analyze its connection to the reward-tilted distribution targeted by prior work, and empirically demonstrate gains over Vanilla BoN, SMC, and prior guidance methods with minimal additional inference cost.

\begin{figure}[t]
\centering
\includegraphics[width=\linewidth]{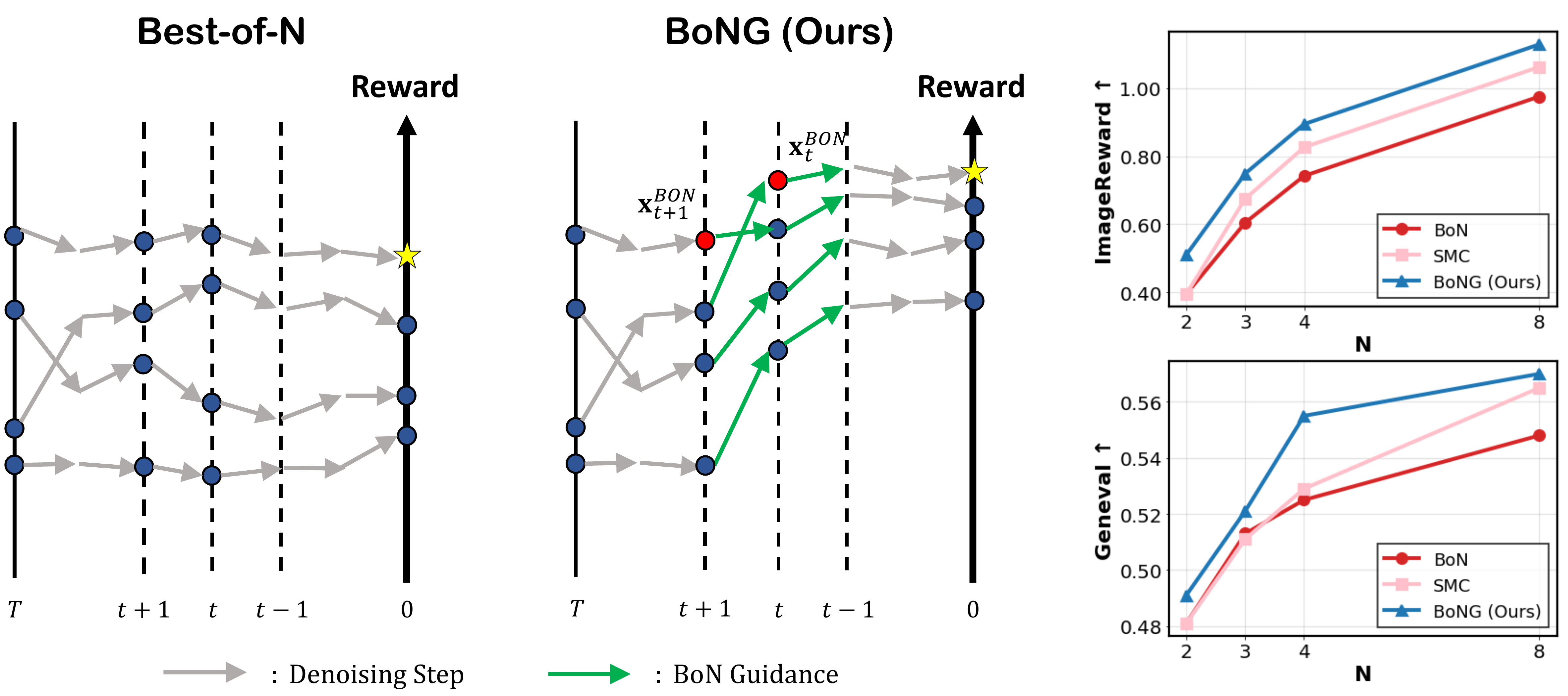}
\small
\caption{
\textbf{Best-of-$N$ Guidance (BoNG)}
(Left) Vanilla BoN sampling selects the highest-reward sample only after generation. 
(Middle) BoNG performs online BoN selection (\cref{alg:bon-select}) and uses the selected sample to guide the remaining particles (\cref{alg:bong-update}). 
(Right) BoNG consistently improves single-output performance over Vanilla BoN and SMC across various sample budgets.}

\label{Fig1}
\end{figure}

\section{Related Work}
\label{Sec:2}
BoN sampling and approaches based on reward-tilted distributions have been widely studied in test-time diffusion alignment. Given a pre-trained diffusion model $p_{\theta}(\rvx_0)$ and a human reward function $r(\rvx_0)$, the reward-tilted distribution is defined as $p^r_{\theta}(\rvx_0)\propto p_{\theta}(\rvx_0)\exp(r(\rvx_0))$. To approximate this distribution, previous work has studied how to estimate the intractable expected future reward $r_t(\rvx_t):=\mathbb{E}_{p_{\theta}(\rvx_0|\rvx_t)} \left[\exp(r(\rvx_0))\right]$ for a noisy sample $\rvx_t$, which induces the intermediate distribution $p_{\theta}^r(\rvx_t) \propto p_{\theta}(\rvx_t)r_t(\rvx_t)$~\cite{uehara2024fine,domingo-enrich2025adjoint}. SMC methods~\cite{singhal2025a,li2025derivativefree,skreta2025feynmankac} generate multiple samples from the base diffusion model and perform resampling based on intermediate reward values. However, SMC suffers from particle collapse, which degrades the quality of the best sample in single-output generation and limits sample diversity in multi-output settings.
Gradient-guidance methods address this issue by enabling i.i.d. sampling by approximating the target time-dependent score function. However, these methods require either training diffusion time-conditioned reward models~\cite{kim2023refining} or performing multiple neural backward passes with Taylor approximation~\cite{bansal2024universal,na2025diffusion} at every denoising step, resulting in significant latency and memory overhead.

Sample-based guidance methods~\cite{potaptchik2025tilt,kim2026lookahead} have recently emerged as a promising direction because they avoid particle collapse and do not require neural backward propagation for guidance. Tilt~\cite{potaptchik2025tilt} proposes a gradient-free method for incorporating expected future rewards into the reverse process. However, it remains computationally expensive because it requires generating multiple samples from $\rvx_t$ to $\rvx_0$ at every timestep $t$, known as backward rollout. GlassFlow~\cite{holderrieth2026glass} improves efficiency by reducing the number of rollout steps through ODE-based sampling, and MFM~\cite{potaptchik2026meta} further accelerates this process by replacing the ODE solver with a distilled model. LiDAR~\cite{kim2026lookahead} proposes a forward rollout strategy for estimating intermediate rewards. This reduces per-step rollout sampling to a single marginal rollout. LiDAR significantly improves sample-based guidance and achieves state-of-the-art performance while being around 10× faster than gradient-guidance methods. However, it still requires additional lookahead sampling, making it 2–3× slower than vanilla diffusion sampling. The goal of this work is to develop an efficient and effective sample-based guidance method inspired by the idea of BoN sampling.

\section{Preliminary}
\subsection{Diffusion Models}
A diffusion model~\cite{sohl2015deep,ho2020denoising,song2021scorebased} consists of a pair of stochastic processes: a forward process and a reverse process. The forward diffusion process injects noise into the data $\rvx_0 \sim p(\rvx_0)$ to produce a perturbed sample $\rvx_t$ via a Gaussian kernel:
\[
p(\rvx_t|\rvx_0)=\mathcal{N}(\rvx_t;\rvx_0,\sigma_t^2\mathbf{I}), 
\]
for every $t \in [0, T]$, with a corresponding gradually increasing noise scale $\sigma_t$. This perturbation produces the marginal distribution $p(\rvx_t)$, and the terminal noise scale $\sigma_T$ is chosen to be sufficiently large so that $p(\rvx_T)$ is close to a pure Gaussian prior.

The reverse process generates a data $\rvx_0$ from a noise $\rvx_T \sim p(\rvx_T)$, and is often represented as a stochastic differential equation~\cite{anderson1982reverse,song2021scorebased}:
\begin{align}
d\rvx_t = \left[\mathbf{f}(\rvx_t, t) - g^2(t)\nabla_{\rvx_t} \log p(\rvx_t)\right] dt + g(t) d\textbf{w}_t, \label{eq:reverse_diff}
\end{align}
where the drift $\mathbf{f}$ and volatility $g$ terms are defined to correspond to the noise schedule $\sigma_t$, and $\textbf{w}_t$ is a standard Wiener process. The diffusion model constructs the reverse process by estimating the score function $\nabla_{\rvx_t} \log p(\rvx_t)\approx\nabla_{\rvx_t} \log p_{\theta}(\rvx_t)$ using a score model trained with denoising score matching~\cite{vincent2011connection,song2019generative} or noise prediction~\cite{ho2020denoising} loss functions.

\subsection{Test-time Human Alignment}

\textbf{Problem Setup:} With the advancement of diffusion models, many pre-trained models are now available as open-source~\cite{rombach2022high,podell2024sdxl,flux2024}, enabling generation conditioned on user input $\rvc$. However, the final generated sample $\rvx_0$ often does not align well with the user input $\rvc$. Accordingly, we assume access to a pre-trained diffusion model $p_{\theta}(\rvx_0|\rvc)$ and a human preference reward function $ r:\mathcal{X}_0 \times \mathcal{C} \rightarrow \mathbb{R}$, and aim to obtain samples $\rvx_0$ such that reward $r(\rvx_0, \rvc)$ is high, by manipulating the inference process. Here, $\mathcal{X}_0$ denotes the data space at time $t=0$, and $\mathcal{C}$ denotes the user's input condition space.  

\subsubsection{Best-of-$N$ Sampling}
\label{Section: Best of N Sampling}

Best-of-$N$ (BoN) sampling is a post-hoc selection procedure that draws \(N\) i.i.d. samples from the base distribution \(p_{\theta}(\rvx_0 \mid \rvc)\) and returns a single sample with the highest reward. 
In finite \(N\), the BoN sample from the base distribution is defined as:
\label{Definition of BoN}
\begin{align}
{\rvx}_{0}^{\mathrm{BoN}}
:=
\underset{\{\rvx_0^i\}_{i=1}^N}{\operatorname{argmax}} \;\; r({\rvx}_0^i,\rvc),
\qquad
{\rvx}_0^i \overset{\text{i.i.d.}}{\sim} p_{\theta}(\rvx_0 \mid \rvc)
\label{eq:bon}
\end{align}
BoN sampling is a remarkably strong inference-time alignment strategy in practice \cite{ichihara2025evaluationbestofnsamplingstrategies}, since it can substantially improve the quality of the final selected sample by increasing $N$, without modifying the underlying generative model. However, it remains fundamentally inefficient: the sampling process is done by \cref{eq:reverse_diff}, reward information is not used during the generation process. As a result, obtaining a sample from a high-reward region may require large $N$. Moreover, BoN sampling is inherently designed to return a single best output, and therefore does not directly address the multi-output setting \cite{microsoft_bing, grok_imagine} in which all \(N\) generated candidates are expected to be of high quality.

\subsubsection{Sampling from a Reward-tilted Distribution}
\label{Section Reward Tilted Distribution}

A line of research~\cite{singhal2025a,holderrieth2026glass,kim2026lookahead} aims to sample from the reward-tilted distribution for alignment defined as:
\begin{align}
p_{\theta}^r(\rvx_0|\rvc) \propto p_{\theta}(\rvx_0|\rvc)\exp(\lambda r(\rvx_0, \rvc)), \label{eq:target}
\end{align}
where $\lambda$ is a scalar hyperparameter that controls the strength of the reward. However, since this goal is defined on the final generated sample $\rvx_0$, it is challenging to determine how to modify the iterative denoising process in \cref{eq:reverse_diff}.

We follow a popular intermediate formulation from prior work~\cite{uehara2024fine,domingo-enrich2025adjoint,holderrieth2026glass} and adopt the score-based formulation of \cite{kim2026lookahead} as:
\begin{align}
\nabla_{\rvx_t} \log p_{\theta}^r(\rvx_t|\rvc)
=
\nabla_{\rvx_t} \log p_{\theta}(\rvx_t|\rvc)
+
\nabla_{\rvx_t}r_t^{\lambda}(\rvx_t, \rvc),
\label{eq:process_target}
\end{align}
where $r_t^{\lambda}(\rvx_t, \rvc) := \log \mathbb{E}_{p_{\theta}(\rvx_0|\rvx_t,\rvc)}\left[\exp(\lambda r(\rvx_0,\rvc))\right]$ is expected future reward at particle $\rvx_t$. Denoising via \cref{eq:reverse_diff} with the score replaced by \cref{eq:process_target} yields samples from the reward-tilted distribution in \cref{eq:target}.
Note that computing $r_t^{\lambda}$ is generally intractable, so most previous work focuses on how to approximate this term efficiently and accurately. Although this reverse-process modification is well suited for multi-output generation, existing approximations typically require costly gradient computation~\cite{bansal2024universal,na2025diffusion} or rollout sampling~\cite{potaptchik2025tilt,kim2026lookahead}, making them less efficient for test-time alignment.

\section{Method}
To mitigate the inefficiency and limited capability in multi-output settings, we refine the idea of Best-of-$N$ sampling as a diffusion-based guidance scheme. A natural approach is to inject the reward information of \({\rvx}_{0}^{\mathrm{BoN}}\) in \cref{eq:bon} directly into the reverse diffusion process, but doing so requires samples from $\mathcal{X}_0$.  Following prior works~\cite{bansal2024universal,na2025diffusion,singhal2025a}, we use one-step denoised samples as an alternative to approximate \({\rvx}_{0}^{\mathrm{BoN}}\), and guide the particle population towards it. We construct our Best-of-$N$ guidance algorithm in two steps.

\subsection{Best-of-$N$ Guidance (BoNG)}

\paragraph{Step 1: Best-of-$N$ target in $\mathcal{X}_0$.}
Let \(\{\rvx_t^i\}_{i=1}^N\) denote the current particle population. For each particle, we compute its one-step denoised sample  
\begin{align}
\hat{\rvx}_0^i
:=
\mathbb{E}_{p_\theta(\rvx_0\mid \rvx_t^i,\rvc)}[\rvx_0]
=
\rvx_t^i+\sigma_t^2\nabla_{\rvx_t^i}\log p_\theta(\rvx_t^i\mid \rvc).
\label{eq:one_step_denoised_sample}
\end{align}
These one-step denoised samples, \(\{\hat{\rvx}_0^i\}_{i=1}^{N}\), define a candidate set in $\mathcal{X}_0$. We then perform Best-of-$N$ selection over this set and define
\begin{align}
i^*=\underset{i}{\operatorname{argmax}}\;\; r(\hat{\rvx}_0^i,\rvc),
\qquad
\hat{\rvx}_0^{\mathrm{BoN}}:=\hat{\rvx}_0^{i^*},
\qquad
\rvx_t^{\mathrm{BoN}}:=\rvx_t^{i^*}.
\label{eq:bon_select}
\end{align}

\paragraph{Step 2: Adjust each particle toward \textbf{\(\hat{\rvx}_0^{\mathrm{BoN}}\)}.}
We define the \textbf{BoN target score} for each particle so that its one-step denoised sample matches \(\hat{\rvx}_0^{\mathrm{BoN}}\). Comparing this target with the one-step denoised sample in \cref{eq:one_step_denoised_sample}, we obtain
\begin{align}
\underbrace{\nabla_{\rvx_t^i}\log p_\theta^{\mathrm{BoN}}(\rvx_t^i\mid \rvc)}_{\textbf{BoN target score}}
:=
\nabla_{\rvx_t^i}\log p_\theta(\rvx_t^i\mid \rvc)
+
\underbrace{\frac{1}{\sigma_t^2}
\left(
\hat{\rvx}_0^{\mathrm{BoN}}-\hat{\rvx}_0^i
\right)}_{:=\,\nabla_{\rvx_t^i}r_t^{\mathrm{BoN}}(\rvx_t^i,\rvc)\ \textbf{(BoN guidance term)}}.
\label{eq:bong_score_def}
\end{align}
The Best-of-$N$ Guidance (BoNG) method is defined by using the \textbf{BoN guidance term} to construct the \textbf{BoN target score} at each timestep $t$. See \cref{alg:bon-select,alg:bong-update}. 


The BoN target score \(\nabla_{\rvx_t^i}\log p_\theta^{\mathrm{BoN}}(\rvx_t^i\mid \rvc)\) in \cref{eq:bong_score_def} admits a simple interpretation.
 For \(\rvx_t^{\mathrm{BoN}}\), the BoN guidance term  \(\nabla_{\rvx_t^i}r_t^{\mathrm{BoN}}(\rvx_t^i,\rvc)\)  vanishes since $\bigl(\hat{\rvx}_0^{\mathrm{BoN}}-\hat{\rvx}_0^{\mathrm{BoN}}\bigr)/\sigma_t^2=0$, and thus \(\rvx_t^{\mathrm{BoN}}\) is denoised according to the original reverse diffusion process without additional guidance.

For any particle \(\rvx_t^i\) with \(i\neq i^*\), the update can be interpreted as the sum of two terms:
\begin{align}
\nabla_{\rvx_t^i}\log p_\theta^{\mathrm{BoN}}(\rvx_t^i\mid\rvc)
=
\underbrace{\nabla_{\rvx_t^{\mathrm{BoN}}}\log p_\theta(\rvx_t^{\mathrm{BoN}}\mid\rvc)}_{\text{BoN Score}}
+
\underbrace{\frac{1}{\sigma_t^2}\bigl(\rvx_t^{\mathrm{BoN}}-\rvx_t^i\bigr)}_{\text{Positive Guidance}}.
\label{eq:8}
\end{align}
This decomposition shows that BoNG induces an asymmetric interaction among particles: each particle \(\rvx_t^i\) with \(i\neq i^*\) is attracted toward the current BoN particle by positive guidance, while simultaneously following the score-induced denoising direction of the BoN trajectory by BoN score. Overall, BoNG embeds the strong selection behavior of Best-of-$N$ sampling via a guidance term without additional feed-forward computations of the score network.

\begin{figure}[t]
    \centering

    \begin{minipage}[t]{0.50\linewidth}
        \centering

        \begin{algorithm}[H]
        \caption{Best-of-$N$ selection}
        \label{alg:bon-select}
        {\fontsize{8.5pt}{9pt}\selectfont
        \begin{algorithmic}
        \STATE \hspace*{-0.7em}\textbf{Input:} $\{\rvx_t^i\}_{i=1}^{N},\, t,\, \rvc,\, p_\theta,\, r$
        \STATE 1. \textbf{for} $i=1,\ldots,N$ \textbf{do}
        \STATE \hspace{\algorithmicindent} $\hat{\rvx}_0^i \leftarrow \mathbb{E}_{p_\theta(\rvx_0 \mid \rvx_t^i,\rvc)}[\rvx_0]$
        \STATE \hspace{1.05em} \textbf{end for}
        \STATE 2. Select $i^*$ according to ~\cref{eq:bon_select}
        \STATE 3. $\rvx_t^{\mathrm{BoN}} \leftarrow \rvx_t^{i^*}$,
               $\hat{\rvx}_0^{\mathrm{BoN}} \leftarrow \hat{\rvx}_0^{i^*}$
        \STATE \hspace*{-0.7em}\textbf{Output:} $\rvx_t^{\mathrm{BoN}},\, \hat{\rvx}_0^{\mathrm{BoN}},\, \{\hat{\rvx}_0^i\}_{i=1}^{N}$
        \end{algorithmic}}
        \end{algorithm}
        \vspace{-0.65cm}
        \begin{algorithm}[H]
        \caption{One-step BoNG update}
        \label{alg:bong-update}
        {\fontsize{8.5pt}{9pt}\selectfont
        \begin{algorithmic}
        \STATE \hspace*{-0.7em}\textbf{Input:} $\{\rvx_t^i\}_{i=1}^{N},\, t,\, \rvc,\, p_\theta,\, r,\, \sigma_t$
        \STATE 1. Obtain $\rvx_t^{\mathrm{BoN}},\, \hat{\rvx}_0^{\mathrm{BoN}},\, \{\hat{\rvx}_0^i\}_{i=1}^{N}$ from Algorithm~\ref{alg:bon-select}
        \STATE 2. \textbf{for} $i=1,\ldots,N$ \textbf{do}
        \STATE \hspace{\algorithmicindent} Update $\rvx_t^i$ to $\rvx_{t-1}^i$ using \cref{eq:bong_score_def}
        \STATE \hspace{1.05em} \textbf{end for}
        \STATE \hspace*{-0.7em}\textbf{Output:} $\{\rvx_{t-1}^i\}_{i=1}^{N}$
        \end{algorithmic}}
        \end{algorithm}
    \end{minipage}
    \hfill
    \begin{minipage}[t]{0.48\linewidth}
        \centering
        \vspace{1em}
        \includegraphics[width=\linewidth]{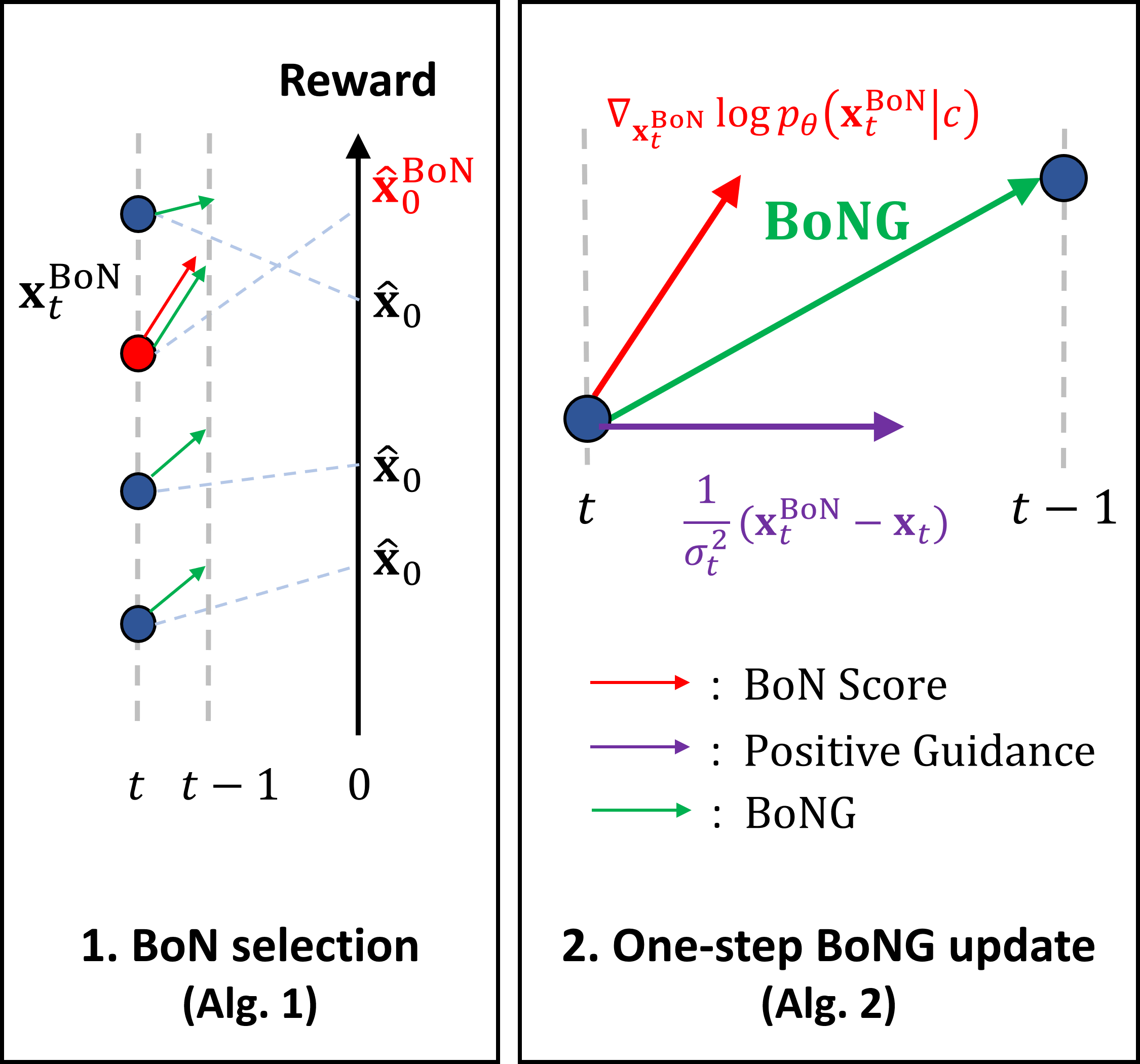}
    \end{minipage}
    \small
    \caption{
    \textbf{Dynamics of BoNG}. (Left) Algorithm Summary. (Right) Decomposition of BoN target score in \cref{eq:bong_score_def}.}
    \label{Fig2}
\end{figure}

\subsection{Connection to Prior Sample-based Test-time Alignment Methods}

We now connect BoNG to prior alignment methods that sample from the reward-tilted distribution in \cref{eq:target}. \cref{lem:fisher} and \cref{lemma1} provide a general analysis of the reward-tilted distribution that holds for any $\lambda>0$: they reduce reward guidance to estimating the tilted posterior mean and give its self-normalized importance sampling (SNIS) estimator. \cref{theorem} then connects this general analysis to BoNG, showing that hard Best-of-$N$ selection realizes the large-$\lambda$ limit of the SNIS estimator, with \cref{alg:bon-select,alg:bong-update} stating the resulting procedure.

\begin{restatable}{lemma}{lemfisher}
\label{lem:fisher}
The gradient of the expected future reward can be expressed as
\begin{align}
\nabla_{\rvx_t}r_t^{\lambda}(\rvx_t,\rvc)
=
\frac{1}{\sigma_t^2}
\left(
m_{\lambda}^{r}(\rvx_t)
-
\mathbb{E}_{p_{\theta}(\rvx_0 \mid \rvx_t,\rvc)}[\rvx_0]
\right),
\qquad
m_{\lambda}^{r}(\rvx_t)
:=
\mathbb{E}_{p_{\theta}^{r}(\rvx_0 \mid \rvx_t,\rvc)}[\rvx_0].
\label{eq:mean_discrepancy}
\end{align}
\end{restatable}

\cref{lem:fisher} follows from the Fisher identity and related score-based formulations~\citep{kim2026lookahead,potaptchik2025tilt}; its role is preparatory. It shows that reward guidance requires no reward derivatives and reduces to estimating the reward-tilted posterior mean \(m_{\lambda}^{r}(\rvx_t)\). Derivation is in \cref{app:lem}.

 We now provide a Monte Carlo approximation scheme of \(m_{\lambda}^{r}(\rvx_t)\) via self-normalized importance sampling~\citep{NIPS2015_39027dfa} (SNIS) by the following lemma, which holds for any \(\lambda>0\).

\begin{restatable}{lemma}{lem}
\label{lemma1}
Let \(q(\rvx_0)\) be any proposal distribution whose support is $\mathcal{X}_0$. Then, given proposal samples \(\hat{\rvx}_0^1,\dots,\hat{\rvx}_0^N \overset{\text{i.i.d.}}{\sim} q(\rvx_0)\), the reward-tilted posterior mean \(m_{\lambda}^{r}(\rvx_t)\) admits the self-normalized importance sampling approximation
\begin{align}
m_{\lambda}^{r}(\rvx_t)
\approx
\hat m_{N,\lambda}^{r}(\rvx_t)
:=
\sum_{i=1}^N \alpha_i^{(\lambda)} \hat{\rvx}_0^i,
\label{eq:snis_rtpm}
\end{align}
where
\[
\alpha_i^{(\lambda)}
=
\operatorname{softmax}_i\!\Bigl(
\log p_\theta(\hat{\rvx}_0^i \mid \rvx_t,\rvc)
-
\log q(\hat{\rvx}_0^i)
+
\lambda r(\hat{\rvx}_0^i,\rvc)
\Bigr).
\]
\end{restatable}

Proof is in \cref{app:lemma1}. \Cref{lemma1} provides a unified view of prior sample-based guidance methods: LiDAR~\citep{kim2026lookahead} corresponds to the choice \(q(\rvx_0)=p_\theta(\rvx_0\mid \rvc)\), while Tilt~\citep{potaptchik2025tilt} uses \(q(\rvx_0)=p_\theta(\rvx_0\mid \rvx_t,\rvc)\). Although these methods differ in their choice of proposal distribution, both approximate \(m_{\lambda}^{r}(\rvx_t)\) through weighted proposal samples. As \(\lambda\to\infty\), the softmax weights concentrate on the highest-reward proposal sample, with \(\alpha_i^{(\lambda)}\to 1\) for the maximizer and \(\alpha_i^{(\lambda)}\to 0\) otherwise, so the SNIS estimator reduces to the highest-reward sample in the proposal set. The following theorem shows that, in this large-\(\lambda\) limit, the process target score in \cref{eq:process_target} coincides with the BoN target score of BoNG in \cref{eq:bong_score_def} when the proposal set is given by the one-step denoised samples \(\{\hat{\rvx}_0^i\}_{i=1}^{N}\) with \(\hat{\rvx}_0^i=\mathbb{E}_{p_\theta(\rvx_0\mid \rvx_t^i,\rvc)}[\rvx_0]\) in \cref{eq:one_step_denoised_sample}.

\begin{figure}[t]
    \centering
    \includegraphics[width=\textwidth]{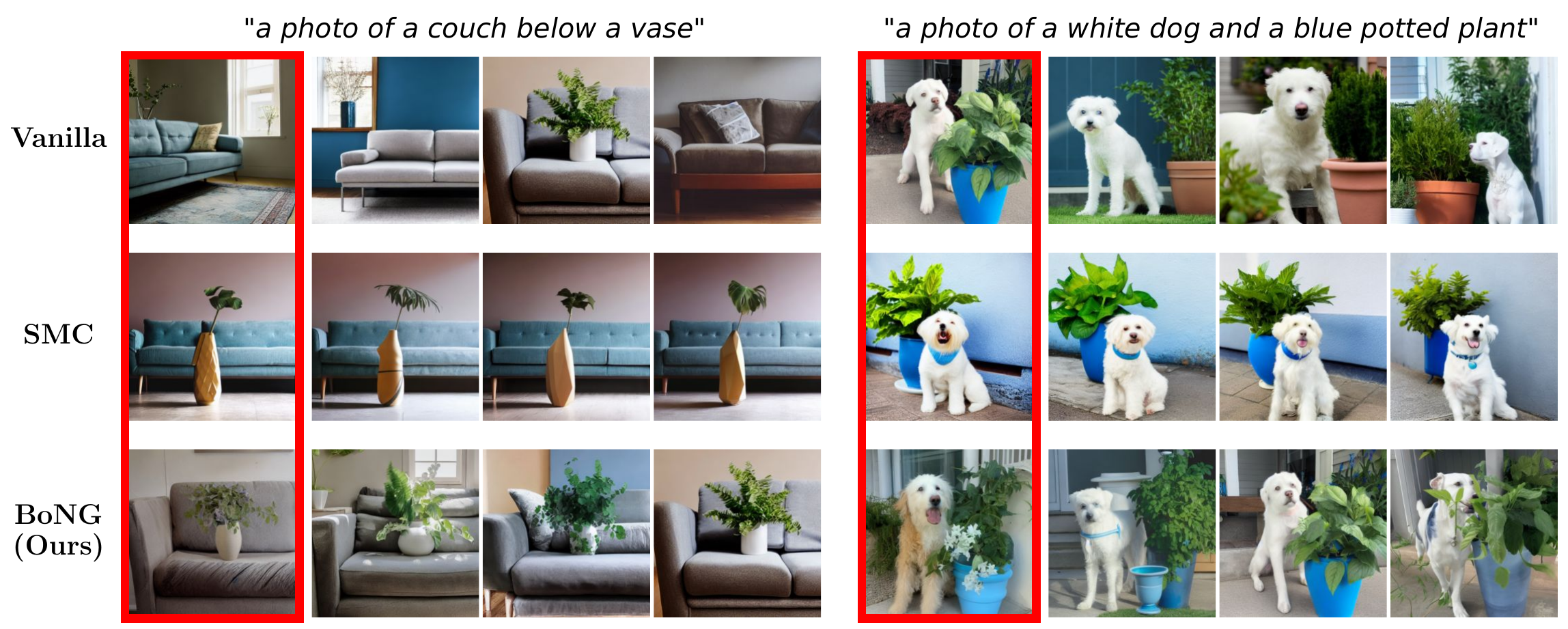}
    \vspace{-2mm}
    \small
    \caption{Qualitative comparison of generated samples from Vanilla, SMC, and BoNG (SD v1.5 DDPM 100 step sampler). \textcolor{red}{Red} boxes denote Best-of-$N$ samples for each method. Vanilla sampling often fails to generate images which align well with specific prompts such as ``below a vase'' and ``a blue potted plant''. SMC generates better aligned images, but resampling leads to premature particle collapse. In contrast, BoNG generates samples that remain well aligned with the prompt while preserving diverse layouts and appearances.}
    \label{fig:samples_compare}
    \vspace{-3mm}
\end{figure}

\begin{restatable}{theorem}{thm}
\label{theorem}
Let \(\{\hat{\rvx}_0^i\}_{i=1}^{N}\) be the proposal set of one-step denoised samples in \cref{eq:one_step_denoised_sample} with a unique maximizer \(i^*=\arg\max_{1\le i\le N} r(\hat{\rvx}_0^i,\rvc)\), and let \(\hat{\rvx}_0^{\mathrm{BoN}}:=\hat{\rvx}_0^{i^*}\). Then \(\hat m_{N,\lambda}^{r}(\rvx_t^i)\to\hat{\rvx}_0^{\mathrm{BoN}}\) as \(\lambda\to\infty\), and the SNIS approximation of the reward-guidance gradient in \cref{eq:mean_discrepancy} converges to the BoN guidance term for each particle:
\begin{align}
\nabla_{\rvx_t^i} r_t^{\lambda}(\rvx_t^i,\rvc)
\;\approx\;
\frac{1}{\sigma_t^2}
\bigl(
\hat m_{N,\lambda}^{r}(\rvx_t^i)-\hat{\rvx}_0^i
\bigr)
\;\xrightarrow{\;\lambda\to\infty\;}\;
\nabla_{\rvx_t^i} r_t^{\mathrm{BoN}}(\rvx_t^i,\rvc).
\label{eq:thm_limit_grad}
\end{align}
\end{restatable}

Proof is in \cref{app:theorem}. \cref{theorem} shows that BoNG realizes the large-\(\lambda\) limit of sample-based guidance toward the reward-tilted distribution, while eliminating the costly forward or backward rollouts of prior methods. Here, the one-step Tweedie estimates serve as a surrogate proposal set rather than exact i.i.d. proposal samples. The same construction also admits a finite-\(\lambda\) relaxation: replacing the hard selection with the soft SNIS weights \(\alpha_i^{(\lambda)}\) of \cref{lemma1} yields a guidance target that interpolates between the average of the proposal set and the BoN target, so that \(\lambda\) directly controls the strength of reward tilting. We refer to \cref{app:finite_lambda} for the full derivation, the approximation gaps of this soft variant, and a \(\lambda\) sweep verifying its predicted reward--diversity scaling.

\subsection{Instantiation of BoNG}
\label{Section: BoNG time window}

BoNG can be applied flexibly during inference time, which allows the method to interpolate between aggressive reward optimization and diversity preservation. Similarly to SMC~\cite{singhal2025a}, we instantiate BoNG by specifying a guidance timestep window, i.e. the subset of reverse diffusion steps at which the BoN guidance term is activated. The guidance window determines how often the particle population is steered toward the current Best-of-$N$ target during reverse diffusion. Applying BoNG over more timesteps yields stronger steering toward high-reward regions, whereas restricting the window leaves more denoising steps governed by the original diffusion process. In this sense, the guidance timestep window in BoNG plays a role analogous to the reward-strength hyperparameter \(\lambda\) in prior methods~\cite{bansal2024universal,na2025diffusion,li2025derivativefree} targeting the reward-tilted distribution in \cref{eq:target}. Consistent with \cref{theorem}, increasing the guidance window corresponds to a more reward-tilted target distribution. In addition, following prior guidance methods~\cite{chung2024cfg++,ho2021classifierfree}, we introduce a guidance scale multiplying the BoN guidance term in \cref{eq:bong_score_def}. This provides an additional degree of control over the magnitude of guidance at each timestep. Further analysis of BoNG is provided in \cref{app:protocol,app:window_effect}.

\section{Experiments}
We use GenEval~\cite{ghosh2023geneval} prompts for evaluation and ImageReward (IR)~\cite{xu2023imagereward} as the reward function during inference time. We report IR, CLIP Score~\cite{hessel2021clipscore}, and GenEval score for all methods. For single-output performance we compare against Vanilla BoN sampling and Sequential Monte Carlo (SMC)~\cite{singhal2025a}, since both methods are designed for best-sample generation under $N$, with comparable computational cost. We report the performance of the BoN sample (i.e., highest IR particle) accordingly. For multi-output performance we compare against representative guidance-based baselines, including gradient-guidance~\cite{bansal2024universal,na2025diffusion} and sample-based guidance~\cite{kim2026lookahead} methods. For each method, we report the average reward of 4 generated samples. See \cref{app:B} for more details.

\newlength{\bonmetricwidth}
\setlength{\bonmetricwidth}{0.85cm}
\newcommand{\bonmetriccell}[1]{\makebox[\bonmetricwidth][c]{#1}}

\renewcommand{\arraystretch}{1.08}
\setlength{\tabcolsep}{3pt}

\begin{table*}[t!]
\small
\centering
\caption{Best-of-$N$ (BoN) performance  on GenEval prompts with different numbers of generated images $N$, using SD v1.5~\cite{rombach2022high} and SDXL~\cite{podell2024sdxl} backbones. ImageReward (IR) is used as the reward for all methods. \textbf{Bold} entries denote the top-performing results within each backbone group.}
\label{Table1}
\vspace{-2mm}

\begin{tabular}{clccccccccc}
\toprule
&
& \multicolumn{9}{c}{\textbf{Performance} $(\uparrow)$} \\
\cmidrule(l){3-11}
\textbf{Backbone}
& \textbf{Method}
& \multicolumn{3}{c}{$\mathnormal{N}=2$}
& \multicolumn{3}{c}{$\mathnormal{N}=3$}
& \multicolumn{3}{c}{$\mathnormal{N}=4$} \\
\cmidrule(l){3-11}
&
& \bonmetriccell{IR} & \bonmetriccell{CLIP} & \bonmetriccell{GenEval}
& \bonmetriccell{IR} & \bonmetriccell{CLIP} & \bonmetriccell{GenEval}
& \bonmetriccell{IR} & \bonmetriccell{CLIP} & \bonmetriccell{GenEval} \\
\midrule
\midrule

\begin{tabular}[c]{@{}c@{}}SD v1.5 (0.9B)\end{tabular}
& Vanilla
& \bonmetriccell{0.317} & \bonmetriccell{0.277} & \bonmetriccell{0.473}
& \bonmetriccell{0.545} & \bonmetriccell{0.281} & \bonmetriccell{0.496}
& \bonmetriccell{0.682} & \bonmetriccell{0.283} & \bonmetriccell{0.520} \\

\begin{tabular}[c]{@{}c@{}}(w/ DDIM 50 steps)\end{tabular}
& SMC
& \bonmetriccell{0.317} & \bonmetriccell{0.277} & \bonmetriccell{0.473}
& \bonmetriccell{0.532} & \bonmetriccell{0.281} & \bonmetriccell{\textbf{0.512}}
& \bonmetriccell{0.636} & \bonmetriccell{0.282} & \bonmetriccell{0.509} \\

& \cellcolor{gray!25}BoNG
& \cellcolor{gray!25}\bonmetriccell{\textbf{0.407}}
& \cellcolor{gray!25}\bonmetriccell{\textbf{0.280}}
& \cellcolor{gray!25}\bonmetriccell{\textbf{0.480}}
& \cellcolor{gray!25}\bonmetriccell{\textbf{0.659}}
& \cellcolor{gray!25}\bonmetriccell{\textbf{0.283}}
& \cellcolor{gray!25}\bonmetriccell{0.507}
& \cellcolor{gray!25}\bonmetriccell{\textbf{0.793}}
& \cellcolor{gray!25}\bonmetriccell{\textbf{0.285}}
& \cellcolor{gray!25}\bonmetriccell{\textbf{0.535}} \\

\midrule

\begin{tabular}[c]{@{}c@{}}SD v1.5 (0.9B)\end{tabular}
& Vanilla
& \bonmetriccell{0.396} & \bonmetriccell{0.279} & \bonmetriccell{0.481}
& \bonmetriccell{0.605} & \bonmetriccell{0.282} & \bonmetriccell{0.513}
& \bonmetriccell{0.742} & \bonmetriccell{0.284} & \bonmetriccell{0.525} \\

\begin{tabular}[c]{@{}c@{}}(w/ DDPM 100 steps)\end{tabular}
& SMC
& \bonmetriccell{0.396} & \bonmetriccell{0.279} & \bonmetriccell{0.481}
& \bonmetriccell{0.675} & \bonmetriccell{0.283} & \bonmetriccell{0.511}
& \bonmetriccell{0.826} & \bonmetriccell{\textbf{0.286}} & \bonmetriccell{0.529} \\

& \cellcolor{gray!25}BoNG
& \cellcolor{gray!25}\bonmetriccell{\textbf{0.511}}
& \cellcolor{gray!25}\bonmetriccell{\textbf{0.280}}
& \cellcolor{gray!25}\bonmetriccell{\textbf{0.491}}
& \cellcolor{gray!25}\bonmetriccell{\textbf{0.749}}
& \cellcolor{gray!25}\bonmetriccell{\textbf{0.284}}
& \cellcolor{gray!25}\bonmetriccell{\textbf{0.521}}
& \cellcolor{gray!25}\bonmetriccell{\textbf{0.894}}
& \cellcolor{gray!25}\bonmetriccell{\textbf{0.286}}
& \cellcolor{gray!25}\bonmetriccell{\textbf{0.555}} \\

\midrule

\begin{tabular}[c]{@{}c@{}}SDXL (2.6B)\end{tabular}
& Vanilla
& \bonmetriccell{0.899} & \bonmetriccell{0.285} & \bonmetriccell{\textbf{0.578}}
& \bonmetriccell{1.060} & \bonmetriccell{0.288} & \bonmetriccell{\textbf{0.600}}
& \bonmetriccell{1.140} & \bonmetriccell{0.289} & \bonmetriccell{\textbf{0.613}} \\

\begin{tabular}[c]{@{}c@{}}(w/ DDIM 50 steps)\end{tabular}
& SMC
& \bonmetriccell{0.899} & \bonmetriccell{0.285} & \bonmetriccell{\textbf{0.578}}
& \bonmetriccell{1.011} & \bonmetriccell{0.287} & \bonmetriccell{0.596}
& \bonmetriccell{1.093} & \bonmetriccell{0.289} & \bonmetriccell{0.612} \\

& \cellcolor{gray!25}BoNG
& \cellcolor{gray!25}\bonmetriccell{\textbf{0.911}}
& \cellcolor{gray!25}\bonmetriccell{\textbf{0.287}}
& \cellcolor{gray!25}\bonmetriccell{0.573}
& \cellcolor{gray!25}\bonmetriccell{\textbf{1.073}}
& \cellcolor{gray!25}\bonmetriccell{\textbf{0.289}}
& \cellcolor{gray!25}\bonmetriccell{0.594}
& \cellcolor{gray!25}\bonmetriccell{\textbf{1.149}}
& \cellcolor{gray!25}\bonmetriccell{\textbf{0.292}}
& \cellcolor{gray!25}\bonmetriccell{0.606} \\

\midrule

\begin{tabular}[c]{@{}c@{}}SDXL (2.6B)\end{tabular}
& Vanilla
& \bonmetriccell{1.037} & \bonmetriccell{0.286} & \bonmetriccell{\textbf{0.600}}
& \bonmetriccell{1.162} & \bonmetriccell{\textbf{0.289}} & \bonmetriccell{0.617}
& \bonmetriccell{1.228} & \bonmetriccell{\textbf{0.290}} & \bonmetriccell{0.634} \\

\begin{tabular}[c]{@{}c@{}}(w/ DDPM 100 steps)\end{tabular}
& SMC
& \bonmetriccell{1.037} & \bonmetriccell{0.286} & \bonmetriccell{\textbf{0.600}}
& \bonmetriccell{1.191} & \bonmetriccell{0.288} & \bonmetriccell{\textbf{0.623}}
& \bonmetriccell{\textbf{1.264}} & \bonmetriccell{\textbf{0.290}} & \bonmetriccell{0.631} \\

& \cellcolor{gray!25}BoNG
& \cellcolor{gray!25}\bonmetriccell{\textbf{1.059}}
& \cellcolor{gray!25}\bonmetriccell{\textbf{0.287}}
& \cellcolor{gray!25}\bonmetriccell{0.597}
& \cellcolor{gray!25}\bonmetriccell{\textbf{1.194}}
& \cellcolor{gray!25}\bonmetriccell{\textbf{0.289}}
& \cellcolor{gray!25}\bonmetriccell{0.618}
& \cellcolor{gray!25}\bonmetriccell{1.259}
& \cellcolor{gray!25}\bonmetriccell{\textbf{0.290}}
& \cellcolor{gray!25}\bonmetriccell{\textbf{0.648}} \\

\midrule
\bottomrule
\end{tabular}

\vspace{-3mm}
\end{table*}

\renewcommand{\arraystretch}{1.0}

\newlength{\avgmetriccolwidth}
\setlength{\avgmetriccolwidth}{0.82cm}
\newcommand{\avgmetriccell}[1]{\makebox[\avgmetriccolwidth][c]{#1}}

\renewcommand{\arraystretch}{1.0}
\setlength{\tabcolsep}{1.6mm}

\begin{table*}[t]
\centering

\begin{minipage}[t]{0.59\textwidth}
\vspace{0pt}
\small
\centering
\caption{Average performance on GenEval prompts with SD v1.5 backbone. Metrics are averaged over the 4 generated images per prompt, and cost is measured for generating 4 images (5 for Best-4-of-5) on a single A100 GPU. \textbf{Bold} entries indicate the best results.}
\vspace{0.2cm}
\label{Table2}
\vspace{-2mm}

\resizebox{\linewidth}{!}{%
\begin{tabular}{lccccc}
\toprule
\textbf{Method}
& \multicolumn{3}{c}{\textbf{Performance} $(\uparrow)$}
& \multicolumn{2}{c}{\textbf{Cost} $(\downarrow)$} \\
\cmidrule(lr){2-4} \cmidrule(l){5-6}
& \avgmetriccell{IR}
& \avgmetriccell{CLIP}
& \avgmetriccell{GenEval}
& \avgmetriccell{Time{\scriptsize~(s)}}
& \avgmetriccell{Mem.{\scriptsize~(G)}} \\
\midrule
\midrule

Vanilla
& \avgmetriccell{-0.001}
& \avgmetriccell{0.271}
& \avgmetriccell{0.426}
& \avgmetriccell{7.07}
& \avgmetriccell{8.90} \\
\cmidrule(lr){1-6}

Vanilla (250 steps)
& \avgmetriccell{0.003}
& \avgmetriccell{0.272}
& \avgmetriccell{0.430}
& \avgmetriccell{17.42}
& \avgmetriccell{\textbf{8.90}} \\

Vanilla Best-4-of-5
& \avgmetriccell{0.180}
& \avgmetriccell{0.275}
& \avgmetriccell{0.453}
& \avgmetriccell{8.83}
& \avgmetriccell{9.97} \\

UG~\cite{bansal2024universal}
& \avgmetriccell{0.326}
& \avgmetriccell{0.262}
& \avgmetriccell{0.355}
& \avgmetriccell{58.36}
& \avgmetriccell{28.16} \\

DATE~\cite{na2025diffusion}
& \avgmetriccell{0.364}
& \avgmetriccell{0.274}
& \avgmetriccell{0.438}
& \avgmetriccell{32.89}
& \avgmetriccell{24.71} \\

LiDAR~\cite{kim2026lookahead}
& \avgmetriccell{0.384}
& \avgmetriccell{0.278}
& \avgmetriccell{0.478}
& \avgmetriccell{13.41}
& \avgmetriccell{\textbf{8.90}} \\

\cellcolor{gray!25}BoNG
& \cellcolor{gray!25}\avgmetriccell{\textbf{0.506}}
& \cellcolor{gray!25}\avgmetriccell{\textbf{0.280}}
& \cellcolor{gray!25}\avgmetriccell{\textbf{0.511}}
& \cellcolor{gray!25}\avgmetriccell{\textbf{8.06}}
& \cellcolor{gray!25}\avgmetriccell{\textbf{8.90}} \\

\midrule
\bottomrule
\end{tabular}%
}

\vspace{-2mm}
\end{minipage}%
\hfill%
\begin{minipage}[t]{0.37\textwidth}
\vspace{0pt}
\centering
\vspace{-0.1cm}
\includegraphics[width=\linewidth]{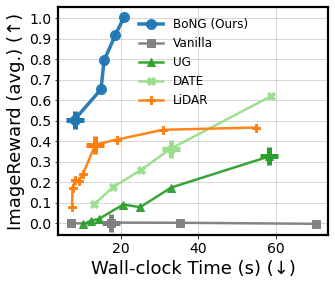}
\vspace{-2mm}

\refstepcounter{figure}
\parbox{\linewidth}{\small
Figure~\thefigure: Reward–time trade-off. The reported time denotes the average time required to generate four images. The cross marker on each curve corresponds to the performance reported in \cref{Table2}.}
\label{Fig3}
\end{minipage}

\vspace{-3mm}
\end{table*}

\renewcommand{\arraystretch}{1.0}

\subsection{Single-Output Performance}
As shown in \cref{Table1}, BoNG consistently outperforms baselines across different backbone models with varying sample budget \(N\). Over 36 comparisons, BoNG achieves the highest score in 29 cases, ranking first in 80.56\% of all comparisons against baselines. Even with 2 particles, BoNG shows steep performance gains over baselines and further shows superior performance across $N=2,3,4$. In particular, BoNG with $N=4$ outperforms Vanilla BoN with $N=8$ in GenEval (see \cref{Fig1}). This advantage scales to larger budgets where BoNG continues to outperform SMC and Vanilla BoN (see \cref{Fig8}). Overall, these results demonstrate that even with small $N$ budget, BoNG outputs a high reward sample compared to Vanilla BoN and SMC. BoNG is also applicable to modern flow-matching backbones, retaining its advantage on FLUX.1-dev (12B)~\cite{flux2024} (\cref{tab:flux}).

\paragraph{Why does BoNG outperform SMC?} SMC also steers the generative process by incorporating reward information by interaction of particles, similar to BoNG. The main difference is that SMC relies on resampling, whereas BoNG does not. In SMC, even if denoising starts from \(N\) independent particles, repeated resampling based on intermediate rewards can collapse the particle population onto only a few trajectories, or even a single one. Such particle collapse reduces exploration and degrades the quality of the final Best-of-$N$ sample. In contrast, BoNG preserves all particles and instead performs Best-of-$N$ selection in the predicted clean-sample space at each intermediate timestep. The selected high-reward target is then used as a guidance signal so that the particle population is steered to high-reward regions without particle copy. As a result, BoNG maintains a better exploration-exploitation tradeoff than SMC, which leads to stronger final BoN performance under the same particle budget. Moreover, since guidance and resampling act on different components of the particle system, the two are composable: applying BoNG on top of SMC improves both the mean and the BoN reward at every particle count (\cref{app:smc_bong}).

\subsection{Multi-Output Performance}
As shown in \cref{Table2}, BoNG achieves the best performance across all three metrics, demonstrating that its benefit is not limited to improving only the final selected sample (BoN) but also for overall samples. More importantly, these gains are obtained with very small test-time overhead. Compared with prior guidance methods, BoNG shows the strongest average performance while remaining close to vanilla diffusion sampling in inference time.

\paragraph{Why does BoNG outperform prior guidance methods?}
These results are consistent with \cref{theorem}, which shows that BoNG can also be interpreted as a sample-based guidance method targeting the reward-tilted distribution \cref{eq:target}. Its strong empirical performance comes from combining this connection with the core strength of Best-of-$N$ sampling: selecting a single high-reward output from \(N\) generated candidates. Rather than applying Best-of-$N$ sampling only after generation, BoNG injects the same principle directly into reverse diffusion by repeatedly using the current Best-of-$N$ target to steer the denoising particle population toward higher-reward regions. At the same time, unlike prior sample-based guidance methods, BoNG does not require additional forward or backward rollouts to construct guidance samples. Instead, it directly reuses the base scores and one-step denoised samples of particles already present in the denoising process. As a result, BoNG inherits the strong best-sample behavior of Best-of-$N$ sampling while avoiding the computational overhead of prior guidance methods, leading to a substantially better efficiency performance tradeoff of the overall quality of multiple generated samples.

\begin{figure*}[t!]
\centering

\begin{minipage}[t]{0.32\textwidth}
\vspace{0pt}
\centering
\includegraphics[width=\linewidth]{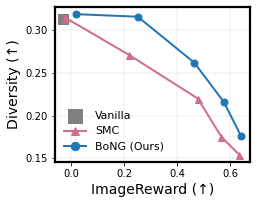}
\vspace{-2mm}

\refstepcounter{figure}
\parbox{\linewidth}{\small
Figure~\thefigure: Reward-Diversity trade-off of BoNG and SMC.} 
\label{Fig4}
\end{minipage}
\hfill%
\begin{minipage}[t]{0.66\textwidth}
\vspace{0pt}
\centering

\begin{minipage}[t]{0.485\linewidth}
    \centering
    \includegraphics[width=\linewidth]{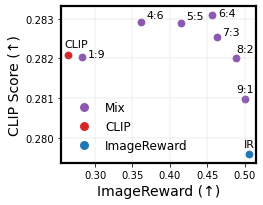}\\[-1mm]
    \makebox[\linewidth][c]{\small\textbf{(a)} Average}
\end{minipage}%
\hfill%
\begin{minipage}[t]{0.485\linewidth}
    \centering
    \includegraphics[width=\linewidth]{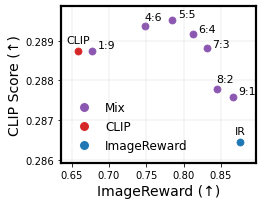}\\[-1mm]
    \makebox[\linewidth][c]{\small\textbf{(b)} BoN}
\end{minipage}

\vspace{1mm}
\refstepcounter{figure}
\parbox{\linewidth}{\small
Figure~\thefigure: Average and BoN scores over four particles under different reward mixing weights (IR : CLIP).}
\label{Fig5}
\end{minipage}%
\hspace{0.015\textwidth}%

\vspace{0mm}
\end{figure*}

\begin{figure*}[t!]
\centering
\begin{minipage}{\textwidth}
\centering
\refstepcounter{table}
\parbox{\linewidth}{\small
Table~\thetable: Guidance with HPSv2~\cite{wu2023human} as the reward model (SD v1.5, DDPM 100 steps, $N=4$).}
\label{tab:hpsv2}
\vspace{-2mm}

{\small
\setlength{\tabcolsep}{5pt}
\begin{tabular}{lcccccccc}
\toprule
& \multicolumn{4}{c}{\textbf{Mean} $(\uparrow)$} & \multicolumn{4}{c}{\textbf{BoN} $(\uparrow)$} \\
\cmidrule(lr){2-5}\cmidrule(l){6-9}
\textbf{Method} & HPSv2 & IR & CLIP & GenEval & HPSv2 & IR & CLIP & GenEval \\
\midrule
Vanilla & 0.263 & $-0.018$ & 0.272 & 0.427 & 0.287 & 0.488 & \textbf{0.282} & \textbf{0.517} \\
\rowcolor{gray!25}BoNG & \textbf{0.276} & \textbf{0.316} & \textbf{0.277} & \textbf{0.487} & \textbf{0.288} & \textbf{0.495} & \textbf{0.282} & 0.514 \\
\bottomrule
\end{tabular}}
\end{minipage}
\vspace{2mm}
\end{figure*}

\begin{figure*}[t!]
\centering
\begin{minipage}[t]{0.34\textwidth}
\vspace{0pt}
\centering
\includegraphics[width=\linewidth]{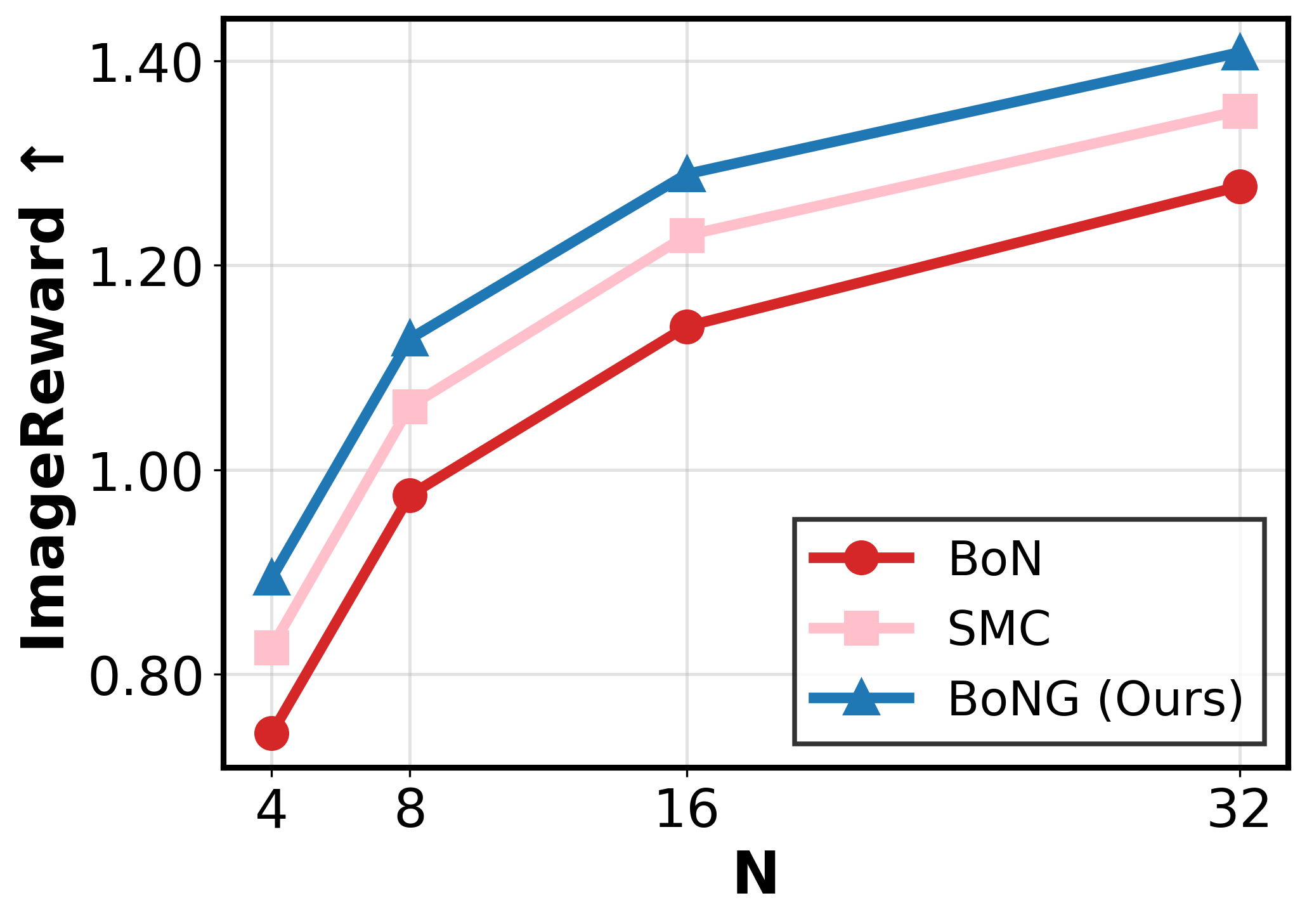}
\vspace{-2mm}

\refstepcounter{figure}
\parbox{\linewidth}{\small
Figure~\thefigure: BoN performance across large sample budget $N$ (SD v1.5, DDPM 100 steps).}
\label{Fig8}
\end{minipage}
\hfill%
\begin{minipage}[t]{0.62\textwidth}
\vspace{0pt}
\centering
\refstepcounter{table}
\parbox{\linewidth}{\small
Table~\thetable: Comparison with LiDAR under matched end-to-end wall-clock time (SD v1.5, DDPM 100 steps, $\delta=5$ for LiDAR; BoN performance). \textbf{Bold} entries indicate the best results.}
\label{tab:lidar_matched}
\vspace{1mm}

{\small
\setlength{\tabcolsep}{4pt}
\begin{tabular}{lcccc}
\toprule
& \multicolumn{3}{c}{\textbf{Performance} $(\uparrow)$} & \textbf{Cost} $(\downarrow)$ \\
\cmidrule(lr){2-4}\cmidrule(l){5-5}
\textbf{Method} & IR & CLIP & GenEval & Time{\scriptsize~(s)} \\
\midrule
LiDAR ($N=2$, $n=50$) & 0.734 & 0.284 & 0.515 & 10.59 \\
LiDAR ($N=3$, $n=25$) & 0.760 & 0.284 & 0.515 & 8.92 \\
LiDAR ($N=4$, $n=9$) & 0.866 & \textbf{0.286} & 0.549 & 8.24 \\
\rowcolor{gray!25}BoNG ($N=4$) & \textbf{0.894} & \textbf{0.286} & \textbf{0.555} & \textbf{8.06} \\
\bottomrule
\end{tabular}}
\end{minipage}
\vspace{1mm}
\end{figure*}

\begin{figure*}[t!]
\centering
\begin{minipage}{\textwidth}
\centering
\refstepcounter{table}
\parbox{\linewidth}{\small
Table~\thetable: BoN performance on FLUX.1-dev (12B)~\cite{flux2024} with the flow-matching Euler 28-step sampler.}
\label{tab:flux}
\vspace{-1mm}

{\small
\setlength{\tabcolsep}{5pt}
\begin{tabular}{lccccccccc}
\toprule
& \multicolumn{3}{c}{$N=2$} & \multicolumn{3}{c}{$N=3$} & \multicolumn{3}{c}{$N=4$} \\
\cmidrule(lr){2-4}\cmidrule(lr){5-7}\cmidrule(l){8-10}
\textbf{Method} & IR & CLIP & GenEval & IR & CLIP & GenEval & IR & CLIP & GenEval \\
\midrule
Vanilla & 1.220 & \textbf{0.285} & 0.671 & 1.303 & \textbf{0.287} & 0.687 & 1.353 & \textbf{0.287} & 0.696 \\
\rowcolor{gray!25}BoNG & \textbf{1.252} & \textbf{0.285} & \textbf{0.680} & \textbf{1.341} & \textbf{0.287} & \textbf{0.696} & \textbf{1.370} & 0.286 & \textbf{0.700} \\
\bottomrule
\end{tabular}}
\end{minipage}
\vspace{-2mm}
\end{figure*}

\subsection{Analysis}
\paragraph{Reward-Diversity Trade-off against SMC.}
\label{Sec:Analysis}
Similar to BoNG, SMC can also be adapted to the multi-output setting. SMC does it by tuning the reward-weighting parameter \(\lambda\) that determines particle potentials, whereas BoNG controls the guidance time window in \cref{Section: BoNG time window}. \cref{Fig4} shows that BoNG achieves a consistently better reward-diversity Pareto frontier than SMC. SMC incorporates reward information through resampling, which necessarily duplicates some particles and removes others, while BoNG incorporates reward information through guidance and thus modifies particle trajectories more gradually. Consequently, BoNG preserves diversity more effectively while maintaining strong reward improvement, giving it a clear advantage in the multi-output setting.

\paragraph{Efficiency–Performance comparison among gradient guidance methods.}
The guidance baselines in \cref{Table2} incur substantial additional computation: UG and DATE require backward passes through the score and reward models, and LiDAR requires additional lookahead rollouts with score-network evaluations (see \cref{app:cost} for the full cost breakdown). BoNG avoids both sources of overhead, yet achieves better average performance than LiDAR. We attribute this to the construction of the guidance signal: LiDAR relies on random lookahead samples, whereas BoNG steers all particles toward the current highest-reward denoising particle. \cref{Fig3} extends the efficiency--performance Pareto points of \cref{Table2}, showing how performance changes as more computation is used; BoNG achieves the steepest performance gain relative to the amount of computation.

\paragraph{Performance with different reward functions.}
\cref{Fig5} shows results when BoNG is guided by CLIP, ImageReward (IR), or their weighted combination. To account for the different reward scales, we Z-score normalize each reward before applying the mixing weights. As the CLIP weight increases, the generated samples achieve higher CLIP scores, whereas increasing the IR weight improves ImageReward scores. This trend shows that BoNG can effectively leverage other reward functions as guidance signals. Furthermore, combining IR and CLIP yields a clear trade-off curve between the two metrics. The mixed rewards produce frontier points, indicating that reward mixing can improve the balance across complementary evaluation metrics. These results suggest that BoNG provides a reward-agnostic guidance mechanism that can operate with different reward functions and can further benefit from combining multiple rewards. We further test HPSv2~\citep{wu2023human} as the guidance reward in \cref{tab:hpsv2}, where BoNG improves both the mean and the BoN scores over Vanilla, confirming that the guidance mechanism transfers to a different human-preference reward. Results with a heavier reward model are provided in \cref{app:hpsv3}.

\paragraph{Budget-matched comparison with LiDAR.}
\cref{Table1} compares methods whose end-to-end cost stays close to that of Vanilla BoN sampling. LiDAR~\citep{kim2026lookahead}, a representative sample-based guidance method, can also serve the single-output setting, but its cost is higher: independently of $N$, it must first generate $n$ lookahead samples in advance, each requiring $\delta$ score-network evaluations, and its performance is governed by $n$ (see \cref{app:cost} for the full cost breakdown). We therefore compare BoNG against LiDAR under matched end-to-end wall-clock time, trading off $N$ against $n$ for LiDAR. As shown in \cref{tab:lidar_matched}, BoNG achieves the best results on every metric with the shortest generation time. This is because BoNG constructs its guidance signal from the $N$ particles currently being sampled, whereas LiDAR relies on separately generated lookahead samples. Additional experiments are provided in \cref{app:lidar_matched}.

\paragraph{Scalability of BoNG.}
BoNG scales with both the sample budget and the backbone. \cref{Fig8} shows that on SD v1.5 (DDPM 100 steps), BoNG consistently outperforms SMC and Vanilla BoN for $N\in\{8,16,32\}$, and its BoN performance keeps improving with $N$, inheriting the scaling behavior of Best-of-$N$ sampling. BoNG also benefits from a stronger generative process: in \cref{Table1}, its performance increases as the backbone grows from SD v1.5 (0.9B) to SDXL (2.6B) and as the sampler becomes more accurate (DDIM 50 to DDPM 100 steps). This trend extends to flow-matching backbones, to which BoNG applies directly since it requires only the one-step denoised estimate in \cref{eq:one_step_denoised_sample}. On FLUX.1-dev (12B)~\citep{flux2024}, \cref{tab:flux} shows that BoNG continues to improve over Vanilla BoN at every $N$, confirming that its gains carry over to a larger and stronger backbone.


\renewcommand{\arraystretch}{1.0}
\setlength{\tabcolsep}{1.6mm}

\section{Conclusion}
\label{Sec.6}
We propose BoNG, a test-time alignment method that integrates Best-of-$N$ sampling directly into the reverse diffusion process of diffusion models. By performing online Best-of-$N$ selection among denoising particles, BoNG steers particles to high reward regions during generation rather than using rewards only for post-hoc selection. We further show that BoNG is theoretically connected to sample-based guidance for reward-tilted diffusion sampling, while avoiding the costly rollout procedure required by prior methods. Empirically, BoNG improves both single-output (BoN) performance and average multi-output quality, achieving strong alignment gains with minimal additional inference cost. These results demonstrate that BoNG is an efficient and flexible framework for test-time diffusion alignment.

\paragraph{Limitations and Future Work.}
\label{limitations}
Effective use of BoNG requires carefully setting the guidance window. Applying BoNG in early denoising steps does not provide a meaningful guidance since the one step denoised samples are noisy and can not provide a meaningful guidance signal. Also, applying BoNG too late can not steer the particle effectively, since the high level features are already determined. For a new backbone model, finding the timestep where BoNG is most effective is required. As future work, the finite-$\lambda$ relaxation in \cref{app:finite_lambda}, which replaces the hard Best-of-$N$ selection with a soft-weighted posterior mean, opens a promising direction: exploring this relaxation in greater depth, for example by replacing one step denoised samples to few step samples would be a meaningful extension of BoNG.

\begin{ack}
This work was supported by the IITP (Institute of Information \& Communications Technology Planning \& Evaluation)-ITRC (Information Technology Research Center) grant funded by the Korea government (Ministry of Science and ICT) (IITP-2026-RS-2024-00437268, 50\%), and by AI Technology Development for Commonsense Extraction, Reasoning, and Inference from Heterogeneous Data (IITP) funded by the Ministry of Science and ICT (RS-2022-II220077, 50\%).
\end{ack}

\bibliography{neurips_2026}
\bibliographystyle{plain}

\appendix

\newpage
\section{Proofs and Derivations}
\subsection{Derivation of Lemma~\ref{lem:fisher}}
\label{app:lem}
\lemfisher*
\begin{proof}    
We have a forward Gaussian Kernel and its Stein score as: 
\[
p_t(\rvx_t|\rvx_0)
=\frac{1}{\sigma_t\sqrt{2\pi}}
\exp\!\left(-\frac{(\rvx_t-\rvx_0)^2}{2\sigma_t^2}\right),
\qquad
\nabla_{\rvx_t}\log p_t(\rvx_t|\rvx_0)
=-\frac{\rvx_t-\rvx_0}{\sigma_t^2}.
\]
With Bayes' rule, we can rewrite expected future reward as follows:

\begin{align}
\mathbb{E}_{p_\theta(\rvx_0 \mid \rvx_t, \rvc)} \left[ \exp(\lambda r(\rvx_0, \rvc)) \right]
&= \int p_{\theta}(\rvx_0|\rvx_t,\rvc) \exp(\lambda r(\rvx_0, \rvc)) \mathrm{d}\rvx_0
\\
&=\int \frac{p_\theta(\rvx_0 \mid \rvc)\, p_t(\rvx_t \mid \rvx_0)}
        {\int p_\theta(\rvx_0 \mid \rvc)\, p_t(\rvx_t \mid \rvx_0)\, \mathrm{d}\rvx_0}  \exp(\lambda r(\rvx_0, \rvc))  \mathrm{d}\rvx_0
\\
&= \frac{\int p_\theta(\rvx_0 \mid \rvc)\, p_t(\rvx_t \mid \rvx_0)\exp(\lambda r(\rvx_0, \rvc))  \mathrm{d}\rvx_0}
        {\int p_\theta(\rvx_0 \mid \rvc)\, p_t(\rvx_t \mid \rvx_0)\, \mathrm{d}\rvx_0}  
\\
&= \frac{z_r(\rvx_t)}{z_0(\rvx_t)},
\end{align}
where 
\begin{align}
z_r(\rvx_t) 
&:= \int p_\theta(\rvx_0 \mid \rvc)\, p_t(\rvx_t \mid \rvx_0)\, \exp(\lambda r(\rvx_0, \rvc))\, \mathrm{d}\rvx_0
\\
z_0(\rvx_t) 
&:= \int p_\theta(\rvx_0 \mid \rvc)\, p_t(\rvx_t \mid \rvx_0)\, \mathrm{d}\rvx_0=p_{\theta}(\rvx_t|\rvc).
\end{align}

The gradient of the expected future reward $\nabla_{\rvx_t}r_t^{\lambda}(\rvx_t,\rvc)$ can be expressed as

\begin{align}
\nabla_{\rvx_t}r_t^{\lambda}(\rvx_t,\rvc) &=
\nabla_{\rvx_t} \log \mathbb{E}_{p_\theta(\rvx_0 \mid \rvx_t, \rvc)} \left[ \exp(\lambda r(\rvx_0, \rvc)) \right]
\\
&= \nabla_{\rvx_t} \log z_r(\rvx_t) - \nabla_{\rvx_t} \log z_0(\rvx_t).
\label{eq:prop_1_ref1}
\end{align}

The first term $\nabla_{\rvx_t} \log z_r(\rvx_t)$ can be expressed as follows:
\begin{align}
    \nabla_{\rvx_t} \log z_r(\rvx_t)&=\frac{\nabla_{\rvx_t}z_r(\rvx_t)}{z_r(\rvx_t)} 
    \\
    &=\int\frac{ p_{\theta}(\rvx_0\mid\rvc)\exp(\lambda r(\rvx_0,\rvc))\nabla_{\rvx_t}p_t(\rvx_t \mid \rvx_0)}{z_r(\rvx_t)}\mathrm{d}\rvx_0
    \\
    &=\int\frac{p_{\theta}(\rvx_0\mid\rvc)\exp(\lambda r(\rvx_0,\rvc))p_t(\rvx_t \mid \rvx_0)\frac{\nabla_{\rvx_t}p_t(\rvx_t \mid \rvx_0)}{p_t(\rvx_t \mid \rvx_0)}}{z_r(\rvx_t)}\mathrm{d}\rvx_0
    \\
    &=\int\frac{ p_{\theta}(\rvx_0\mid\rvc)p_t(\rvx_t \mid \rvx_0)}{z_0(\rvx_t)}\frac{z_0(\rvx_t)}{z_r(\rvx_t)}\exp(\lambda r(\rvx_0,\rvc))\nabla_{\rvx_t}\log p_t(\rvx_t \mid \rvx_0)\mathrm{d}\rvx_0
    \\
    &=\int \underbrace{\frac{p_{\theta}(\rvx_0\mid\rvx_t,\rvc)\exp(\lambda r(\rvx_0,\rvc))}{\mathbb{E}_{p_{\theta}(\rvx_0\mid\rvx_t,\rvc)}\left[\exp(\lambda r(\rvx_0,\rvc))\right]}}_{:=p^r_{\theta}(\rvx_0 \mid \rvx_t,\rvc) }\nabla_{\rvx_t}\log p_t(\rvx_t \mid \rvx_0) \mathrm{d}\rvx_0
    \\
    &=\mathbb{E}_{p^r_{\theta}(\rvx_0 \mid \rvx_t,\rvc)}\left[ \nabla_{\rvx_t}\log p_t(\rvx_t \mid \rvx_0)\right].
\end{align}
We can rewrite the second term $\nabla_{\rvx_t} \log z_0(\rvx_t)$ as follows:
\begin{align}
    \nabla_{\rvx_t} \log z_0(\rvx_t)&=\frac{\nabla_{\rvx_t}z_0(\rvx_t)}{z_0(\rvx_t)} 
    \\
    &=\int\frac{ p_{\theta}(\rvx_0\mid\rvc)\nabla_{\rvx_t}p_t(\rvx_t \mid \rvx_0)}{z_0(\rvx_t)}\mathrm{d}\rvx_0
    \\
    &=\int \frac{ p_{\theta}(\rvx_0\mid\rvc)p_t(\rvx_t \mid \rvx_0)\frac{\nabla_{\rvx_t}p_t(\rvx_t \mid \rvx_0)}{p_t(\rvx_t \mid \rvx_0)}}{z_0(\rvx_t)}\mathrm{d}\rvx_0
    \\
    &=\int \frac{ p_{\theta}(\rvx_0\mid\rvc)p_t(\rvx_t \mid \rvx_0)}{p_{\theta}(\rvx_t|\rvc)}\nabla_{\rvx_t}\log p_t(\rvx_t \mid \rvx_0)\mathrm{d}\rvx_0
    \\
    &=\mathbb{E}_{p_{\theta}(\rvx_0|\rvx_t,\rvc)}\left[\nabla_{\rvx_t}\log p_t(\rvx_t \mid \rvx_0)\right].
\end{align}

Substituting the two identities above into Eq.~\ref{eq:prop_1_ref1} and simplifying yields
\begin{align}
\nabla_{\rvx_t}r_t^{\lambda}(\rvx_t,\rvc)
&=\mathbb{E}_{p^r_{\theta}(\rvx_0|\rvx_t,\rvc)}\left[\nabla_{\rvx_t}\log p_t(\rvx_t \mid \rvx_0)\right] - \mathbb{E}_{p_{\theta}(\rvx_0|\rvx_t,\rvc)}\left[\nabla_{\rvx_t}\log p_t(\rvx_t \mid \rvx_0)\right]
\\
&=\mathbb{E}_{p^r_\theta(\rvx_0|\rvx_t,\rvc)}
\left[
-\frac{\rvx_t - \rvx_0}{\sigma_t^2}
\right] - \mathbb{E}_{p_\theta(\rvx_0|\rvx_t,\rvc)}
\left[
-\frac{\rvx_t - \rvx_0}{\sigma_t^2}
\right]
\\
&=-\frac{1}{\sigma_t^2}
\left(
\rvx_t - \mathbb{E}_{p^r_\theta(\rvx_0|\rvx_t,\rvc)}[\rvx_0]
\right) +\frac{1}{\sigma_t^2}
\left(
\rvx_t - \mathbb{E}_{p_\theta(\rvx_0|\rvx_t,\rvc)}[\rvx_0]
\right)
\\
&=
\frac{1}{\sigma_t^2}
\left(
\mathbb{E}_{p^r_\theta(\rvx_0|\rvx_t,\rvc)}[\rvx_0]
-
\mathbb{E}_{p_\theta(\rvx_0|\rvx_t,\rvc)}[\rvx_0]
\right)
\\
&=
\frac{1}{\sigma_t^2}
\left(
m_{\lambda}^{r}(\rvx_t)
-
\mathbb{E}_{p_\theta(\rvx_0|\rvx_t,\rvc)}[\rvx_0]
\right).
\end{align}
This completes the proof.
\end{proof}

\subsection{Derivation of Lemma~\ref{lemma1}}
\label{app:lemma1}
\lem*

\begin{proof}    
We derive the self-normalized importance sampling representation of the reward-tilted posterior mean. Specifically, we obtain
\begin{align}
m_{\lambda}^{r}(\rvx_t)
&=
\mathbb{E}_{p_{\theta}^{r}(\rvx_0 \mid \rvx_t,\rvc)}[\rvx_0]
=\frac{\int p_{\theta}(\rvx_0 \mid \rvx_t,\rvc)\exp (\lambda r(\rvx_0,\rvc))\rvx_0\mathrm{d}\rvx_0}{\mathbb{E}_{p_{\theta}(\rvx_0 \mid \rvx_t,\rvc)}\left[\exp (\lambda r(\rvx_0,\rvc))\right]}
\\
&=
\frac{
\mathbb{E}_{p_{\theta}(\rvx_0 \mid \rvx_t,\rvc)}
\left[
\rvx_0 \exp\!\bigl(\lambda r(\rvx_0,\rvc)\bigr)
\right]
}{
\mathbb{E}_{p_{\theta}(\rvx_0 \mid \rvx_t,\rvc)}
\left[
\exp\!\bigl(\lambda r(\rvx_0,\rvc)\bigr)
\right]
}.
\label{eq:appendix_rtpm_def}
\end{align}

Let \(q(\rvx_0)\) be any proposal distribution whose support contains that of \(p_{\theta}(\rvx_0 \mid \rvx_t,\rvc)\). By change of measure, both the numerator and denominator in \cref{eq:appendix_rtpm_def} can be rewritten under \(q(\rvx_0)\), yielding
\begin{align}
m_{\lambda}^{r}(\rvx_t)
=
\frac{
\mathbb{E}_{q(\rvx_0)}
\left[
\frac{p_{\theta}(\rvx_0 \mid \rvx_t,\rvc)}{q(\rvx_0)}
\rvx_0 \exp\!\bigl(\lambda r(\rvx_0,\rvc)\bigr)
\right]
}{
\mathbb{E}_{q(\rvx_0)}
\left[
\frac{p_{\theta}(\rvx_0 \mid \rvx_t,\rvc)}{q(\rvx_0)}
\exp\!\bigl(\lambda r(\rvx_0,\rvc)\bigr)
\right]
}.
\label{eq:appendix_rtpm_q}
\end{align}

Now let \(\hat{\rvx}_0^1,\dots,\hat{\rvx}_0^N \overset{i.i.d.}{\sim} q(\rvx_0)\). Applying Monte Carlo approximation to both expectations in \cref{eq:appendix_rtpm_q}, we obtain
\begin{align}
m_{\lambda}^{r}(\rvx_t)
&\approx
\frac{
\frac{1}{N}\sum_{i=1}^{N}
\frac{p_{\theta}(\hat{\rvx}_0^i \mid \rvx_t,\rvc)}{q(\hat{\rvx}_0^i)}
\hat{\rvx}_0^i
\exp\!\bigl(\lambda r(\hat{\rvx}_0^i,\rvc)\bigr)
}{
\frac{1}{N}\sum_{i=1}^{N}
\frac{p_{\theta}(\hat{\rvx}_0^i \mid \rvx_t,\rvc)}{q(\hat{\rvx}_0^i)}
\exp\!\bigl(\lambda r(\hat{\rvx}_0^i,\rvc)\bigr)
}
\\
&=
\frac{
\sum_{i=1}^{N}
\frac{p_{\theta}(\hat{\rvx}_0^i \mid \rvx_t,\rvc)}{q(\hat{\rvx}_0^i)}
\hat{\rvx}_0^i
\exp\!\bigl(\lambda r(\hat{\rvx}_0^i,\rvc)\bigr)
}{
\sum_{i=1}^{N}
\frac{p_{\theta}(\hat{\rvx}_0^i \mid \rvx_t,\rvc)}{q(\hat{\rvx}_0^i)}
\exp\!\bigl(\lambda r(\hat{\rvx}_0^i,\rvc)\bigr)
}.
\label{eq:appendix_mc_ratio}
\end{align}

Define the unnormalized importance weights
\begin{align}
w_i^{(\lambda)}
:=
\frac{p_{\theta}(\hat{\rvx}_0^i \mid \rvx_t,\rvc)}{q(\hat{\rvx}_0^i)}
\exp\!\bigl(\lambda r(\hat{\rvx}_0^i,\rvc)\bigr).
\end{align}
Then \cref{eq:appendix_mc_ratio} can be written as
\begin{align}
m_{\lambda}^{r}(\rvx_t)
\approx
\sum_{i=1}^{N}
\frac{w_i^{(\lambda)}}{\sum_{j=1}^{N} w_j^{(\lambda)}}
\hat{\rvx}_0^i.
\end{align}
Therefore, if we define
\begin{align}
\alpha_i^{(\lambda)}
:=
\frac{w_i^{(\lambda)}}{\sum_{j=1}^{N} w_j^{(\lambda)}},
\end{align}
we obtain the self-normalized importance sampling estimator
\begin{align}
\hat m_{N,\lambda}^{r}(\rvx_t)
:=
\sum_{i=1}^{N} \alpha_i^{(\lambda)} \hat{\rvx}_0^i.
\end{align}

Finally, taking logarithms of the unnormalized weights gives
\begin{align}
\log w_i^{(\lambda)}
=
\log p_{\theta}(\hat{\rvx}_0^i \mid \rvx_t,\rvc)
-
\log q(\hat{\rvx}_0^i)
+
\lambda r(\hat{\rvx}_0^i,\rvc).
\end{align}
Hence the normalized weights admit the softmax form
\begin{align}
\alpha_i^{(\lambda)}
=
\operatorname{softmax}_i\!\Bigl(
\log p_{\theta}(\hat{\rvx}_0^i \mid \rvx_t,\rvc)
-
\log q(\hat{\rvx}_0^i)
+
\lambda r(\hat{\rvx}_0^i,\rvc)
\Bigr),
\end{align}
which proves the result.
\end{proof}

\subsection{Derivation of Theorem~\ref{theorem}}
\label{app:theorem}
\thm*

\begin{proof}
By \cref{eq:mean_discrepancy}, the target process score in \cref{eq:process_target} depends on the reward-tilted posterior mean \(m_{\lambda}^{r}(\rvx_t)\). By \cref{lemma1}, this posterior mean admits the self-normalized importance sampling approximation
\[
m_{\lambda}^{r}(\rvx_t)\approx \hat m_{N,\lambda}^{r}(\rvx_t)
=
\sum_{i=1}^{N}\alpha_i^{(\lambda)}\hat{\rvx}_0^i,
\]
where
\[
\alpha_i^{(\lambda)}
=
\operatorname{softmax}_i\!\Bigl(
\log p_\theta(\hat{\rvx}_0^i \mid \rvx_t,\rvc)
-
\log q(\hat{\rvx}_0^i)
+
\lambda r(\hat{\rvx}_0^i,\rvc)
\Bigr).
\]

Let
\[
i^*=\arg\max_{1\le i\le N} r(\hat{\rvx}_0^i,\rvc),
\qquad
\hat{\rvx}_0^{\mathrm{BoN}}:=\hat{\rvx}_0^{i^*},
\]
which is unique by assumption. Then for any \(i\neq i^*\),
\begin{align}
\frac{\alpha_i^{(\lambda)}}{\alpha_{i^*}^{(\lambda)}}
&=
\exp\Bigl(
\log p_\theta(\hat{\rvx}_0^i \mid \rvx_t,\rvc)
-
\log q(\hat{\rvx}_0^i)
+
\lambda r(\hat{\rvx}_0^i,\rvc)
\Bigr) \nonumber
\\
&\qquad\times
\exp\Bigl(
-\log p_\theta(\hat{\rvx}_0^{i^*} \mid \rvx_t,\rvc)
+
\log q(\hat{\rvx}_0^{i^*})
-
\lambda r(\hat{\rvx}_0^{i^*},\rvc)
\Bigr) \nonumber
\\
&=
C_i\exp\Bigl(
\lambda\bigl(r(\hat{\rvx}_0^i,\rvc)-r(\hat{\rvx}_0^{i^*},\rvc)\bigr)
\Bigr),
\end{align}
where
\[
C_i
:=
\exp\Bigl(
\log p_\theta(\hat{\rvx}_0^i \mid \rvx_t,\rvc)
-
\log q(\hat{\rvx}_0^i)
-
\log p_\theta(\hat{\rvx}_0^{i^*} \mid \rvx_t,\rvc)
+
\log q(\hat{\rvx}_0^{i^*})
\Bigr)
\]
is independent of \(\lambda\). Since \(i^*\) is the unique maximizer, we have
\[
r(\hat{\rvx}_0^i,\rvc)-r(\hat{\rvx}_0^{i^*},\rvc)<0
\qquad
\text{for all } i\neq i^*.
\]

Therefore,
\[
\frac{\alpha_i^{(\lambda)}}{\alpha_{i^*}^{(\lambda)}}\xrightarrow[\lambda\to\infty]{}0
\qquad
\text{for all } i\neq i^*.
\]
Since \(\sum_{i=1}^{N}\alpha_i^{(\lambda)}=1\), it follows that
\[
\alpha_i^{(\lambda)}\xrightarrow[\lambda\to\infty]{}0
\quad (i\neq i^*),
\qquad
\alpha_{i^*}^{(\lambda)}\xrightarrow[\lambda\to\infty]{}1.
\]
Hence,
\begin{align}
\hat m_{N,\lambda}^{r}(\rvx_t)
=
\sum_{i=1}^{N}\alpha_i^{(\lambda)}\hat{\rvx}_0^i
\xrightarrow[\lambda\to\infty]{}
\hat{\rvx}_0^{i^*}
=
\hat{\rvx}_0^{\mathrm{BoN}}.
\label{eq:appendix_theorem3_collapse}
\end{align}

Now evaluate at a query particle \(\rvx_t=\rvx_t^i\). By the proposal construction in \cref{eq:one_step_denoised_sample,eq:bon_select},
\[
\mathbb{E}_{p_\theta(\rvx_0\mid \rvx_t^i,\rvc)}[\rvx_0]=\hat{\rvx}_0^i,
\qquad
\hat{\rvx}_0^{\mathrm{BoN}}=\hat{\rvx}_0^{i^*}.
\]
Substituting the SNIS approximation into \cref{eq:mean_discrepancy} evaluated at \(\rvx_t^i\), we obtain
\begin{align}
\nabla_{\rvx_t^i}r_t^{\lambda}(\rvx_t^i,\rvc)
\approx
\frac{1}{\sigma_t^2}
\bigl(
\hat m_{N,\lambda}^{r}(\rvx_t^i)-\hat{\rvx}_0^i
\bigr).
\label{eq:appendix_theorem3_guidance}
\end{align}
By \cref{eq:appendix_theorem3_collapse}, the right-hand side converges as \(\lambda\to\infty\) to
\[
\frac{1}{\sigma_t^2}
\bigl(
\hat{\rvx}_0^{\mathrm{BoN}}-\hat{\rvx}_0^i
\bigr)
=
\nabla_{\rvx_t^i} r_t^{\mathrm{BoN}}(\rvx_t^i,\rvc),
\]
which is the BoN guidance term in \cref{eq:bong_score_def}. This establishes \cref{eq:thm_limit_grad} and proves the theorem.
\end{proof}

\section{Experimental Settings}
\label{app:B}

\subsection{Experiment Setup}
\label{sec:B.1}
We use SD v1.5~\citep{rombach2022high} and SDXL~\citep{podell2024sdxl} backbones with CFG~\citep{ho2021classifierfree} scale 7.5 for all experiments. For each backbone, we experiment with the DDPM 100-step sampler and the DDIM 50-step sampler (probability-flow ODE) in \cref{Table1}, the latter being substantially faster. We additionally evaluate on the flow-matching backbone FLUX.1-dev~\citep{flux2024}, using the flow-matching Euler sampler with 28 steps and guidance scale 3.5. All results are averaged over three random seeds, except FLUX.1-dev, which uses a single seed due to computational cost.

\subsection{BoNG Configuration}
\label{Sec:B.2}
BoNG has three hyperparameters: the window position $t_{\mathrm{start}}$, the window size $K$, and the guidance scale $s$. The selection protocol and the corresponding sensitivity analyses are given in \cref{app:protocol,app:C}; here we record the values used in the main experiments.

For a compute budget comparable to the SMC baseline~\citep{singhal2025a}, which performs five resampling steps, we set the number of guided transitions to $K\le4$ in all configurations of \cref{Table1}. A guidance window $[t_{\mathrm{start}},t_{\mathrm{end}}]$ denotes the set of guided transitions $t\to t+1$ for $t_{\mathrm{start}}\le t<t_{\mathrm{end}}$; e.g., $[40,44]$ denotes the four transitions $40\to41,\dots,43\to44$ ($K=4$). The guidance windows are $[40,44]$ for SD v1.5 with DDPM 100 steps (where step 40 corresponds to $t=0.6$, with $t=1$ the prior and $t=0$ the data), $[21,23]$ for SD v1.5 with DDIM 50 steps, $[24,28]$ for SDXL with DDPM 100 steps, $[20,22]$ for SDXL with DDIM 50 steps, $[12,14]$ for FLUX with Euler 28 steps. The guidance scale is $s=6$ for SD v1.5, $s=4$ for SDXL, $s=2$ for FLUX. \cref{Table2} uses the same configuration as \cref{Table1}.

\subsection{Baseline Configuration}
\label{sec:B.3}
We apply the shared settings of \cref{sec:B.1} to all baselines. In \cref{Table1}, Vanilla runs standard diffusion sampling $N$ times and selects the best sample according to ImageReward~\citep{xu2023imagereward}. The baseline denoted SMC is FK-Steering~\citep{singhal2025a} under its recommended text-to-image configuration: max potential, $\lambda=10$, five evenly spaced resampling steps, adaptive resampling, and intermediate rewards on denoised estimates. Best-$k$-of-$N$ involves no guidance: $N=5$ samples are generated with the base model, and the top $k=4$ samples by reward are returned.

For \cref{Table2}, we follow the standard practice of quoting the reported numbers for LiDAR~\citep{kim2026lookahead}, UG~\citep{bansal2024universal}, and DATE~\citep{na2025diffusion} from the concurrent work LiDAR, which achieves state-of-the-art performance in the multi-output setting. Please refer to Sections B.1 and B.2 of LiDAR for the corresponding configurations. For transparency, we additionally reproduce all baselines of \cref{Table2} in our unified codebase under the same configurations and report the results in \cref{app:reproduced_table2}.

\subsection{Evaluation Metrics}
\label{sec:B.4}
We report the metrics used throughout our experiments.

\noindent\textbf{ImageReward}~\citep{xu2023imagereward}:
a learned human-preference reward model for text-to-image generation. It assigns a scalar preference score to each generated image conditioned on the text prompt, and we use it to evaluate overall image quality and prompt satisfaction from a human-centric perspective.

\noindent\textbf{CLIP Score}~\citep{hessel2021clipscore}:
measures the semantic alignment between the text prompt and the generated image by computing the similarity between their CLIP embeddings. It provides a reference-free estimate of prompt-image correspondence, although it may be less sensitive to fine-grained compositional errors.

\noindent\textbf{GenEval}~\citep{ghosh2023geneval}:
an object-centric evaluation benchmark for text-to-image generation. It uses detection-based checks to assess whether generated images correctly satisfy prompt-specified entities, attributes, counts, and spatial relations, making it useful for measuring compositional accuracy.

We measure the diversity of generated samples by computing pairwise distances in the CLIP image-embedding space, following the evaluation protocol of FK-Steering.\footnote{\url{https://github.com/zacharyhorvitz/D-Diffusion-Steering/blob/main/text_to_image/fkd_diffusers/rewards.py}}

\section{Hyperparameter Selection Protocol of BoNG}
\label{app:protocol}

BoNG has three hyperparameters: the window position $t_{\mathrm{start}}$, the window size $K$, and the guidance scale $s$. For a new backbone or sampler, we set them by the following protocol:
\begin{itemize}
    \item \textbf{Window position} $t_{\mathrm{start}}$: place the window where the correlation between the Tweedie-estimate reward $r(\hat{\rvx}_0)$ and the final reward becomes reliable ($\approx0.6$).
    \item \textbf{Window size} $K$: choose $K$ by the compute budget, at or below the resampling budget of the SMC baseline.
    \item \textbf{Guidance scale} $s$: tune $s\in[1,6]$ to select the desired point on the reward--diversity frontier, at no extra compute.
\end{itemize}
This procedure was followed for every configuration in the paper. The remainder of this section justifies each rule with controlled experiments: \cref{app:protocol_tstart} shows why the window position trades off the accuracy of the guidance signal against the steerability of the samples, \cref{app:protocol_k} shows that the window size moves gradually and monotonically along the reward--diversity frontier, and \cref{app:protocol_s} shows that the guidance scale sweeps the same frontier at zero additional cost.

\subsection{Window Position}
\label{app:protocol_tstart}
A meaningful guidance signal requires the one-step denoised estimate $\hat{\rvx}_0^{\mathrm{BoN}}$ to be aligned with the final reward, so we place the window near the point where the correlation between $r(\hat{\rvx}_0)$ and the final reward exceeds $0.6$ (cf. Figure 5 of FK-Steering~\citep{singhal2025a}). Placing the window too late is ineffective for the opposite reason: by then the high-level features of each sample are already determined, so particles committed to low-reward regions can no longer be steered.

\begin{table}[H]
\small
\centering
\caption{Ablation over the window position $t_{\mathrm{start}}$ (SD v1.5, DDPM 100 steps, $N=4$, $K=4$, $s=6$; BoN performance).}
\vspace{1.5mm}
\label{tab:tstart_ablation}
\begin{tabular}{lccc}
\toprule
\textbf{Window} & \textbf{IR} $(\uparrow)$ & \textbf{CLIP} $(\uparrow)$ & \textbf{GenEval} $(\uparrow)$ \\
\midrule
Vanilla   & 0.713 & 0.284 & 0.532 \\
$[20,24]$ & 0.753 & 0.282 & 0.523 \\
\rowcolor{gray!25}$[40,44]$ & \textbf{0.875} & 0.286 & \textbf{0.559} \\
$[60,64]$ & 0.849 & \textbf{0.287} & 0.526 \\
$[80,84]$ & 0.866 & \textbf{0.287} & 0.528 \\
\bottomrule
\end{tabular}
\end{table}

\cref{tab:tstart_ablation} reports an ablation over $t_{\mathrm{start}}$ with all other factors fixed (SD v1.5, DDPM 100 steps, $N=4$, $K=4$, $s=6$, single seed; BoN performance). The metric-wise breakdown reveals why the optimum sits mid-range: window placement trades off the \textit{accuracy} of the guidance signal against the \textit{steerability} of the samples. An early window ($[20,24]$) fails on both metrics because the Tweedie-reward correlation is still low. In a late window ($[80,84]$), the correlation is at its highest and ImageReward improves accordingly; however, the high-level compositional content that GenEval measures is already determined by that stage, so accurate guidance can refine samples but can no longer re-steer particles committed to compositionally weaker configurations, and GenEval does not improve over Vanilla. The mid-range window $[40,44]$, where the correlation first becomes reliable while the high-level content can still be shaped, is the only regime that improves all metrics simultaneously, which is precisely what the window-position rule above is designed to find.

\subsection{Window Size}
\label{app:protocol_k}
The window size $K$ determines the additional cost of BoNG, which scales linearly as $NK(D+R)$ (see \cref{app:cost}). Enlarging $K$ applies the reward signal over more transitions and correspondingly reduces the influence of the base diffusion process, concentrating the particles in high-reward regions; in this sense, $K$ plays a role analogous to the reward-strength hyperparameter $\lambda$ in prior methods~\citep{bansal2024universal,na2025diffusion,li2025derivativefree}. \cref{tab:k_scan} reports a controlled scan over $K$ (SD v1.5, DDPM 100 steps, $s=1$, $t_{\mathrm{start}}=40$). Both effects are gradual and monotone: reward increases and diversity decreases as $K$ grows. At the main-experiment setting $K=4$, the reward improves over Vanilla sampling with no loss of diversity (0.319 vs.\ 0.313), indicating that the main experiments operate in a regime free of particle collapse. To keep the compute budget below that of the SMC baseline, we set $K=4$ for 100-step sampling and $K=2$ for 50-step sampling.

\begin{table}[H]
\small
\centering
\caption{Controlled scan over the window size $K$ (SD v1.5, DDPM 100 steps, $N=4$, $s=1$, $t_{\mathrm{start}}=40$).}
\vspace{1.5mm}
\label{tab:k_scan}
\begin{tabular}{lccc}
\toprule
\textbf{Window} & $K$ & \textbf{mean IR} $(\uparrow)$ & \textbf{Diversity} $(\uparrow)$ \\
\midrule
Vanilla   & --   & $-0.001$ & 0.313 \\
\rowcolor{gray!25}$[40,44]$ & 4    & 0.019 & 0.319 \\
$[40,50]$ & 10   & 0.252 & 0.316 \\
$[40,55]$ & 15   & 0.465 & 0.262 \\
$[40,60]$ & 20   & 0.576 & 0.216 \\
$[40,65]$ & 25   & 0.641 & 0.176 \\
$[40,70]$ & 30   & 0.678 & 0.146 \\
$[40,75]$ & 35   & 0.707 & 0.121 \\
\bottomrule
\end{tabular}
\end{table}

\subsection{Guidance Scale}
\label{app:protocol_s}
CFG~\citep{ho2021classifierfree} improves image fidelity without increasing inference time or the number of function evaluations (NFEs) by introducing a guidance scale $\omega\in\mathbb{R}$ with $\omega\geq 1$ into the conditional score function, as defined in \cref{eq:cfg}. In the same spirit, following prior guidance methods that expose an analogous scale~\citep{bansal2024universal,kim2026lookahead}, we introduce a guidance scale $s\in\mathbb{R}$ with $s\geq1$ on the BoN guidance term, which controls how strongly each non-BoN particle (\(\rvx_t^i\) with \(i\neq i^*\)) is steered, as described in \cref{eq:guidance_scale}.

\begin{align}
    \nabla_{\rvx_t}\log{p_\theta^\omega(\rvx_t|\rvc)}=\nabla_{\rvx_t}\log{p_\theta(\rvx_t)}+\omega\times\frac{\mathbb{E}_{p_\theta(\rvx_0|\rvx_t,\rvc)}[\rvx_0]-\mathbb{E}_{p_\theta(\rvx_0|\rvx_t)}[\rvx_0]}{\sigma_t^2}
    \label{eq:cfg}
\end{align}

\begin{align}
    \nabla_{\rvx_t}\log{p_\theta^{\mathrm{BoN}}(\rvx_t|\rvc)}
    =
    \nabla_{\rvx_t}\log{p_\theta(\rvx_t|\rvc)}
    +s\times
    \frac{
    \mathbb{E}_{p_\theta(\rvx_0|\rvx_t^{\mathrm{BoN}},\rvc)}[\rvx_0]
    -
    \mathbb{E}_{p_\theta(\rvx_0|\rvx_t,\rvc)}[\rvx_0]
    }{\sigma_t^2}
    \label{eq:guidance_scale}
\end{align}

Within the protocol, the scale $s$ complements $K$: since it only rescales the guidance term, it sweeps the same reward--diversity frontier at zero additional compute. \cref{tab:s_scan} reports a scan at the fixed main-experiment window $[40,44]$ ($K=4$): increasing $s$ raises the mean reward while diversity decreases gradually, so $s$ selects the operating point on the frontier once $t_{\mathrm{start}}$ and $K$ are fixed.

\begin{table}[H]
\small
\centering
\caption{Scan over the guidance scale $s$ (SD v1.5, DDPM 100 steps, $N=4$, window $[40,44]$, $K=4$).}
\vspace{1.5mm}
\label{tab:s_scan}
\begin{tabular}{lcc}
\toprule
$s$ & \textbf{mean IR} $(\uparrow)$ & \textbf{Diversity} $(\uparrow)$ \\
\midrule
Vanilla & $-0.001$ & 0.313 \\
1 & 0.019 & 0.319 \\
2 & 0.224 & 0.315 \\
3 & 0.490 & 0.255 \\
4 & 0.473 & 0.245 \\
\rowcolor{gray!25}6 & 0.506 & 0.224 \\
\bottomrule
\end{tabular}
\end{table}

We further compare the sensitivity of BoNG to $s$ against UG and LiDAR, which expose the same hyperparameter, on GenEval~\citep{ghosh2023geneval} test prompts. \cref{Fig6} reports the ImageReward--CLIP score trade-off as $s$ varies, and we make three observations. First, in the range $s\in[1,7.5]$, UG exhibits reward-hacking behavior: as $s$ increases, ImageReward improves while the CLIP score decreases. In contrast, LiDAR and BoNG achieve more favorable trade-offs across the two metrics. Second, BoNG achieves a stronger ImageReward--CLIP score trade-off than LiDAR. Finally, when $s=10$, we observe a degradation of BoNG in CLIP score due to issues such as off-manifold steering and color saturation, consistent with the findings of CFG++~\citep{chung2024cfg++} and APG~\citep{sadat2024eliminating}; this motivates the protocol range $s\in[1,6]$. Qualitative results are shown in \cref{Fig7}, where we present generations for five randomly selected prompts as $s$ varies from $1$ to $10$.

\begin{figure}[h]
    \centering
    \includegraphics[width=0.6\linewidth]{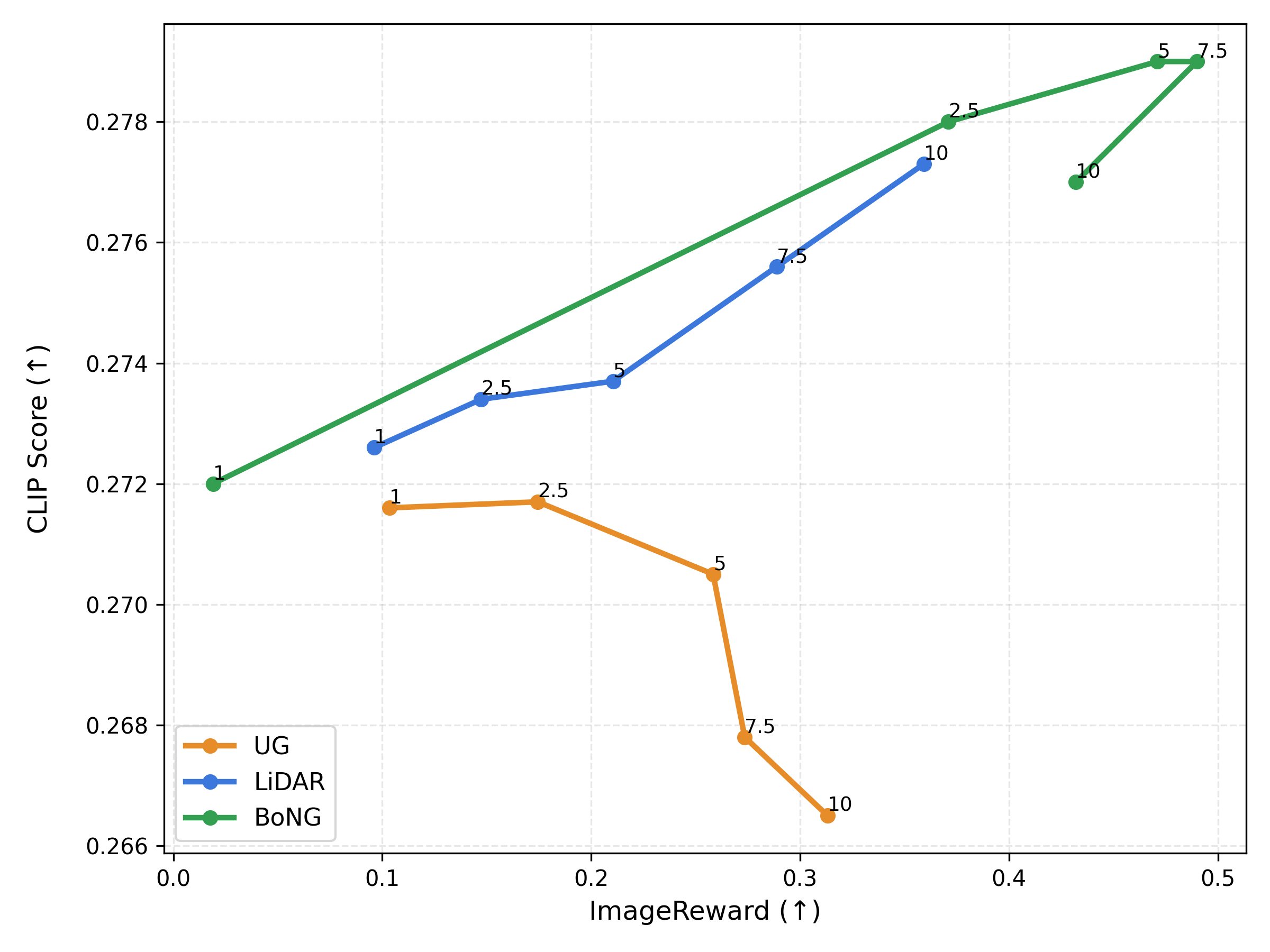}
    \caption{
    ImageReward~\citep{xu2023imagereward} and CLIP score~\citep{hessel2021clipscore} trade-offs between UG, LiDAR, and BoNG. Across all methods, we report results under guidance scale $s\in\{1,2.5,5,7.5,10\}$.}
    \label{Fig6}
\end{figure}

\begin{figure}[!t]
    \centering
    \begin{tikzpicture}
        \node[anchor=south west, inner sep=0] (img) at (0,0)
        {\includegraphics[width=\linewidth]{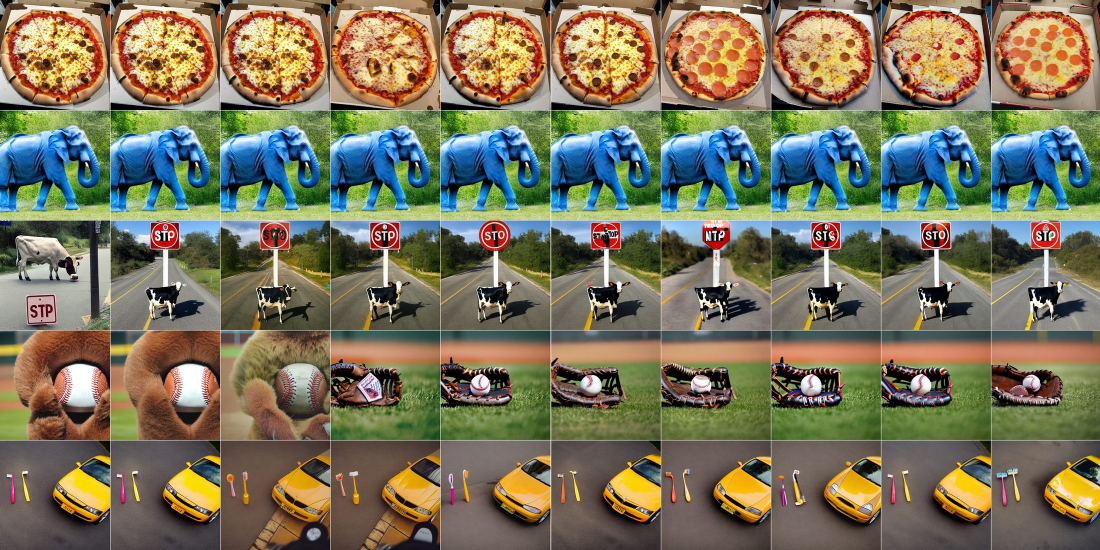}};
        
        \draw[->, line width=1.2pt]
            ([yshift=0.25cm]img.north west) -- ([yshift=0.25cm]img.north east);

        \node[anchor=south] at ([yshift=0.32cm]img.north)
            {\small Guidance scale};
    \end{tikzpicture}
    \caption{
    Changes in Best-of-$N$ samples of BoNG as the guidance scale $s\in\{1,2,3,4,5,6,7,8,9,10\}$ varies.
    The prompts from the top row to the bottom row are ``\textit{a photo of a pizza}'', ``\textit{a photo of a blue elephant}'', ``\textit{a photo of a cow left of a stop sign}'', ``\textit{a photo of a baseball glove right of a bear}'', and ``\textit{a photo of a yellow car and an orange toothbrush}''.
    }
    \label{Fig7}
\end{figure}

\section{Additional Analysis of BoNG}
\label{app:C}

\subsection{Effect of Guidance Timestep Window}
\label{app:window_effect}
The guidance timestep window jointly determines reward optimization, diversity preservation, and inference-time overhead. As discussed in \cref{Section: BoNG time window,Sec:Analysis}, the window plays the role of $\lambda$ in the reward-tilted distribution: applying BoNG over more timesteps yields stronger reward optimization, but may reduce sample diversity by making the generation process closer to sampling from a highly reward-tilted distribution. At the same time, the window size sets the additional cost $NK(D+R)$ of BoNG (\cref{app:cost}). \cref{Fig4} shows the resulting reward--diversity trade-off, \cref{tab:k_scan} confirms that both effects are gradual and monotone, and \cref{Fig3} shows that BoNG achieves the best performance--efficiency trade-off among the compared methods; the selection rules for the window are given in \cref{app:protocol}.

\subsection{Computational Cost}
\label{app:cost}
BoNG introduces no extra score network evaluations over Vanilla diffusion sampling. Its additional cost comes only from constructing the BoN target at guided timesteps, which requires a VAE decoder pass and a reward model pass for each denoising particle. For $N$ particles and $K$ guided timesteps, letting $D$ and $R$ denote the cost of one VAE decoder and reward model forward pass, respectively, the additional function evaluations are $NK(D+R)$.

\cref{tab:nfe} compares each method's additional cost over Vanilla sampling in hardware-independent units. SMC and BoNG add no additional score-network evaluations; their overhead comes solely from scoring denoised estimates, $NS(D+R)$ and $NK(D+R)$, respectively. UG and DATE evaluate the decoder and reward at every step, $NT(D+R)$, and additionally require backward passes through the score network, decoder, and reward model to obtain reward gradients. LiDAR generates lookahead samples: each of the $n$ lookaheads costs $\delta$ score-network evaluations plus one decoder and one reward pass, adding $n\delta + n(D+R)$ in total. Since $K\le S\ll T$ in all our experiments ($K=4$, $S=5$), BoNG incurs the smallest additional cost among the steering methods, with no gradient computation.

\begin{table}[H]
\small
\centering
\caption{Additional cost over Vanilla sampling in hardware-independent units. $D$: VAE-decoder pass, $R$: reward pass; $T$: diffusion steps; $S$: resampling steps of SMC; $n$, $\delta$: lookahead samples and score evaluations per lookahead of LiDAR.}
\vspace{1.5mm}
\label{tab:nfe}
\setlength{\tabcolsep}{5pt}
\begin{tabular}{lcccc}
\toprule
\textbf{Method} & \textbf{Score-net} & \textbf{VAE decoder} & \textbf{Reward} & \textbf{Backward} \\
\midrule
SMC & -- & $N\cdot S$ & $N\cdot S$ & -- \\
UG / DATE & -- & $N\cdot T$ & $N\cdot T$ & required \\
LiDAR & $n\cdot\delta$ & $n$ & $n$ & -- \\
\rowcolor{gray!25}BoNG & -- & $N\cdot K$ & $N\cdot K$ & -- \\
\bottomrule
\end{tabular}
\end{table}

To make the comparisons concrete, \cref{tab:wallclock} reports the wall-clock breakdown of one guided denoising step on SD v1.5 ($N=16$, mean of 5 runs).

\begin{table}[H]
\small
\centering
\caption{Wall-clock breakdown of one guided denoising step (SD v1.5, $N=16$, mean of 5 runs).}
\vspace{1.5mm}
\label{tab:wallclock}
\begin{tabular}{lccccc}
\toprule
\textbf{Component} & \textbf{Time (ms)} & \textbf{Per particle (ms)} & \textbf{Time \%} & \textbf{TFLOPs} & \textbf{TFLOPs / particle} \\
\midrule
UNet fwd + CFG & 238.88 & 14.93 & 20.6\% & 25.705 & 1.607 \\
VAE decode & 522.98 & 32.69 & 45.1\% & 40.232 & 2.515 \\
Postprocess (tensor$\to$PIL) & 167.94 & 10.50 & 14.5\% & 0.000 & 0.000 \\
Reward (ImageReward) & 230.81 & 14.43 & 19.9\% & 2.205 & 0.138 \\
\bottomrule
\end{tabular}
\end{table}

Regarding memory, we report peak GPU memory during generation in GiB in \cref{Table2}. All methods share an identical configuration (backbone, resolution, $N=4$, GPU), so the differences do not stem from image size or encoding. Notably, sample count alone does not predict the measurements: UG and DATE store no extra samples yet require roughly $3\times$ the memory of Vanilla, while LiDAR stores $n=50$ lookahead latents yet matches Vanilla's footprint. We instead characterize the measurements through the memory decomposition itself. With the same networks resident across methods, peak memory is approximately $M \approx P_{\mathrm{score}} + P_{\mathrm{dec}} + P_{\mathrm{reward}} + N\cdot A$, where only the activation term $A$ differs: gradient guidance (UG, DATE) retains the activation graph for backpropagation ($A_{\mathrm{bwd}}$), whereas Vanilla, LiDAR, and BoNG discard activations layer by layer ($A_{\mathrm{fwd}}\ll A_{\mathrm{bwd}}$). BoNG retains no graph and allocates nothing beyond the $N$ particles of Vanilla sampling, so its memory footprint equals Vanilla's by construction, independently of hardware.

\subsection{Anchor-Switching Ablation}
\label{app:anchor}
To test whether a single particle dominates the guidance signal, we tracked the identity of the Best-of-$N$ anchor $\rvx_t^{\mathrm{BoN}}$ at every guided transition (window $[40,44]$: three consecutive transitions per prompt) across all 553 GenEval prompts (total of 1,659 transitions) across 3 runs. We report the mean, median, and standard deviation of the switch rate of the 3 runs. Switch rate $1.0$ implies that for every transition $\rvx_t^{\mathrm{BoN}}$ was changed, and $0$ implies that $\rvx_t^{\mathrm{BoN}}$ is not changed at all. The experiment shows that the anchor is dynamically re-elected from the population rather than fixed: no single trajectory dominates the guidance target, and the guidance signal genuinely emerges from particle interaction.

\begin{table}[H]
\small
\centering
\caption{Anchor-switching statistics over 553 GenEval prompts (1,659 transitions), 3 runs.}
\vspace{1.5mm}
\label{tab:anchor_switch}
\begin{tabular}{lc}
\toprule
\textbf{Anchor switching (553 prompts, 1,659 transitions)} & \textbf{Value} \\
\midrule
Mean & 0.2656 \\
Median & 0.2737 \\
Std & 0.0161 \\
\bottomrule
\end{tabular}
\end{table}

\section{Finite-\texorpdfstring{$\lambda$}{lambda} Soft Relaxation of BoNG}
\label{app:finite_lambda}

BoNG applies hard Best-of-$N$ selection over the one-step denoised candidates and therefore corresponds to the large-$\lambda$ limit described in \cref{theorem}. However, \cref{lemma1} itself holds for every finite $\lambda>0$, and the same rollout-free one-step Tweedie construction admits a natural finite-$\lambda$ relaxation, which we derive and evaluate in this section. We emphasize that this relaxation is a variant beyond the hard-selection algorithm evaluated in the main experiments; all results in the main text correspond to the hard-selection BoNG.

\paragraph{Derivation.}
Choosing the proposal distribution as $q(\rvx_0)=p_\theta(\rvx_0\mid\rvc)$, the SNIS weights of \cref{lemma1} for a query particle $\rvx_t^i$ can be written, using Bayes' rule, as
\begin{align}
\mathrm{softmax}_j\Bigl(\log p_t\bigl(\rvx_t^i \mid \rvx_0^j\bigr) + \lambda\, r\bigl(\rvx_0^j,\rvc\bigr)\Bigr),
\label{eq:finite_lambda_exact}
\end{align}
where $\rvx_0^j\sim p_\theta(\rvx_0\mid\rvc)$ and $p_t(\rvx_t\mid\rvx_0)$ is the tractable Gaussian forward kernel; the finite-$\lambda$ estimator is thus computable in principle when exact proposal samples are available. Obtaining such samples, however, requires full reverse-process rollouts. To construct a rollout-free relaxation, we replace the exact samples with the one-step Tweedie estimates already available from the current particle population, $\hat{\rvx}_0^j = \mathbb{E}_{p_\theta(\rvx_0\mid\rvx_t^j,\rvc)}[\rvx_0]$, which gives the soft weights
\begin{align}
\tilde\alpha_{ij}^{(\lambda)}
= \mathrm{softmax}_j\Bigl(\log p_t\bigl(\rvx_t^i \mid \hat{\rvx}_0^j\bigr) + \lambda\, r\bigl(\hat{\rvx}_0^j,\rvc\bigr)\Bigr),
\label{eq:finite_lambda_soft}
\end{align}
and the corresponding particle-specific relaxed guidance target
\begin{align}
\hat m^{r}_{N,\lambda}(\rvx_t^i) = \sum_{j=1}^{N}\tilde\alpha_{ij}^{(\lambda)}\,\hat{\rvx}_0^j.
\label{eq:finite_lambda_target}
\end{align}
Rather than selecting only the highest-reward candidate, this target softly combines the candidate set according to both the forward compatibility with the query particle and the reward. Smaller values of $\lambda$ distribute the guidance across multiple candidates, while larger values increasingly concentrate the weights on the highest-reward candidate; assuming a unique reward maximizer, the hard-selection BoNG update is recovered as $\lambda\to\infty$, consistent with \cref{theorem}.

\paragraph{Approximation gaps.}
This construction remains approximate in two respects. First, the one-step Tweedie estimates are deterministic posterior means rather than exact proposal samples from $p_\theta(\rvx_0\mid\rvc)$. Second, once particle interaction has been applied, the guided particles need not exactly follow the original marginal $p_\theta(\rvx_t\mid\rvc)$. Consequently, the exact importance-sampling interpretation of \cref{lemma1} does not directly transfer to this rollout-free relaxation: BoNG admits a rollout-free finite-$\lambda$ relaxation, rather than exactly targeting the finite-$\lambda$ reward-tilted distribution.

\paragraph{Empirical verification of the predicted scaling.}
The finite-$\lambda$ analysis predicts that $\lambda$ traces the reward--diversity trade-off, with hard-selection BoNG recovered at the large-$\lambda$ end of the spectrum. To verify this, we implement the soft-weighted relaxation, replacing hard Best-of-$N$ selection with the softmax weights $\tilde\alpha_{ij}^{(\lambda)}$ over the one-step denoised candidates, with the same guidance window as the main experiments (SD v1.5, DDPM 100 steps). \cref{tab:finite_lambda} reports the mean ImageReward and diversity as $\lambda$ varies. Increasing $\lambda$ raises the mean reward while diversity decreases gradually, which is precisely the behavior the finite-$\lambda$ analysis predicts. These results suggest a soft-weighted family of BoNG variants beyond the hard-selection algorithm evaluated in this paper, which we view as a natural direction for future work opened up by our framework.

\begin{table}[H]
\small
\centering
\setlength{\tabcolsep}{5pt}
\caption{Soft-weighted finite-$\lambda$ relaxation of BoNG under varying $\lambda$ ($N=4$).}
\label{tab:finite_lambda}
\vspace{1.5mm}
\begin{tabular}{clcc}
\toprule
\textbf{Backbone} & \textbf{Method} & \textbf{Mean IR} $(\uparrow)$ & \textbf{Diversity} $(\uparrow)$ \\
\midrule
\multirow{7}{*}{\shortstack{SD v1.5 (0.9B)\\(w/ DDPM 100 steps)}}
& Vanilla & $-0.001$ & 0.313 \\
\cmidrule(l){2-4}
& $\lambda=100$ & 0.001 & 0.313 \\
& $\lambda=500$ & 0.004 & 0.311 \\
& $\lambda=1000$ & 0.131 & 0.300 \\
& $\lambda=5000$ & 0.564 & 0.234 \\
& $\lambda=10000$ & 0.595 & 0.214 \\
\bottomrule
\end{tabular}
\end{table}


\section{Additional Experiment Results}
\label{app:add_exp}
This section supplements the main results along both evaluation axes. For the multi-output setting of \cref{Table2}, \cref{app:reproduced_table2} provides a unified-codebase reproduction with the SMC entry and diversity added, and \cref{app:smc_bong} shows that BoNG composes with SMC. For the single-output setting of \cref{Table1}, \cref{app:lidar_matched} extends the matched wall-clock comparison against LiDAR in \cref{Sec:Analysis} to a large sample budget, and \cref{app:hpsv3} tests a heavier guidance reward beyond the HPSv2 result in \cref{tab:hpsv2}.

\subsection{Unified-Codebase Reproduction of Table 2}
\label{app:reproduced_table2}
For transparency, we reproduce all baselines of \cref{Table2} in our unified codebase, following the LiDAR paper's configurations for the UG and DATE baselines, and additionally report the SMC entry and the diversity of each method. \cref{tab:reproduced_table2} reports the results. The relative ordering of \cref{Table2} is preserved under reproduction. SMC attains a higher mean reward than BoNG in this table, but at a substantially lower diversity (0.155 vs.\ 0.224), since resampling duplicates high-scoring ancestors; this makes it unsuitable for the multi-output setting that \cref{Table2} targets, where $N$ distinct outputs must be served.

\begin{table}[H]
\small
\centering
\setlength{\tabcolsep}{4pt}
\caption{Unified-codebase reproduction of \cref{Table2} with the SMC entry and diversity added.}
\vspace{1.5mm}
\label{tab:reproduced_table2}
\begin{tabular}{lcccccc}
\toprule
\textbf{Method} & \textbf{IR} $(\uparrow)$ & \textbf{CLIP} $(\uparrow)$ & \textbf{GenEval} $(\uparrow)$ & \textbf{Diversity} $(\uparrow)$ & \textbf{Time}{\scriptsize~(s)} $(\downarrow)$ & \textbf{Mem.}{\scriptsize~(GiB)} $(\downarrow)$ \\
\midrule
Vanilla & $-0.001$ & 0.271 & 0.426 & 0.313 & 7.07 & \textbf{8.9} \\
\midrule
UG & 0.271 & 0.261 & 0.364 & \textbf{0.461} & 57.89 & 28.2 \\
DATE & 0.382 & 0.275 & 0.450 & 0.327 & 32.68 & 24.7 \\
LiDAR ($\delta=5$, $n=50$) & 0.394 & 0.278 & 0.473 & 0.280 & 13.57 & \textbf{8.9} \\
SMC & \textbf{0.641} & \textbf{0.283} & \textbf{0.517} & 0.155 & 8.23 & \textbf{8.9} \\
\rowcolor{gray!25}BoNG & 0.506 & 0.280 & 0.511 & 0.224 & \textbf{8.06} & \textbf{8.9} \\
\bottomrule
\end{tabular}
\end{table}

\subsection{Comparison with LiDAR under Matched Wall-Clock Budgets}
\label{app:lidar_matched}
\cref{tab:lidar_matched} in \cref{Sec:Analysis} compares BoNG with LiDAR under matched end-to-end wall-clock time at $N=4$. Here we extend the comparison to a large sample budget. At $N=32$, we set LiDAR to $n=100$ and $\delta=5$ so that its additional cost $n\delta+n(D+R)$ consumes similar time as the $NK(D+R)$ cost of BoNG (see \cref{app:cost}), and compare the average performance over the $N$ samples in \cref{tab:large_n_budget}. BoNG outperforms LiDAR on every metric with a shorter end-to-end time. This is because LiDAR's performance gain is independent of $N$ and relies on $n$, whereas BoNG inherits the scaling behavior of Best-of-$N$ sampling as $N$ increases.

\begin{table}[H]
\small
\centering
\setlength{\tabcolsep}{5pt}
\caption{Budget-matched comparison with LiDAR at $N=32$. \textbf{Bold} entries indicate the best results.}
\label{tab:large_n_budget}
\vspace{1.5mm}
\begin{tabular}{clcccc}
\toprule
& & \multicolumn{3}{c}{\textbf{Performance} $(\uparrow)$} & \textbf{Cost} $(\downarrow)$ \\
\cmidrule(lr){3-5}\cmidrule(l){6-6}
\textbf{Backbone} & \textbf{Method} & IR & CLIP & GenEval & Time{\scriptsize~(s)} \\
\midrule
\multirow{4}{*}{\shortstack{SD v1.5 (0.9B)\\(w/ DDPM 100 steps)}}
& Vanilla & $-0.012$ & 0.272 & 0.425 & 49.5 \\
\cmidrule(l){2-6}
& LiDAR ($n=100$, $\delta=5$) & 0.405 & 0.281 & 0.478 & 61.4 \\
& \cellcolor{gray!25}BoNG & \cellcolor{gray!25}\textbf{0.854} & \cellcolor{gray!25}\textbf{0.285} & \cellcolor{gray!25}\textbf{0.538} & \cellcolor{gray!25}\textbf{56.1} \\
\bottomrule
\end{tabular}
\end{table}

\subsection{Performance with a Heavier Reward Model}
\label{app:hpsv3}
\cref{tab:hpsv2} in \cref{Sec:Analysis} shows that BoNG transfers to HPSv2~\cite{wu2023human} as the guidance reward. To further test whether the guidance mechanism holds under a heavier reward backbone, we run Vanilla, SMC, and BoNG with HPSv3~\cite{ma2025hpsv3widespectrumhumanpreference} as the guidance reward ($N=4$, SD v1.5, DDPM 100 steps) and report the BoN performance in \cref{tab:hpsv3}. BoNG improves the HPSv3 reward over Vanilla, confirming that the guidance mechanism transfers to a heavier reward backbone, while SMC attains a slightly higher HPSv3 score in this setting. We also observe that ImageReward decreases under HPSv3 guidance for both SMC and BoNG; we attribute this to HPSv3 using a non-CLIP encoder, so that the two reward models emphasize different features.

\begin{table}[H]
\small
\centering
\setlength{\tabcolsep}{6pt}
\caption{Guidance with HPSv3 as the reward model ($N=4$; BoN performance). \textbf{Bold} entries indicate the best results.}
\label{tab:hpsv3}
\begin{tabular}{clcccc}
\toprule
\textbf{Backbone} & \textbf{Method} & \textbf{HPSv3} $(\uparrow)$ & \textbf{IR} $(\uparrow)$ & \textbf{CLIP} $(\uparrow)$ & \textbf{GenEval} $(\uparrow)$ \\
\midrule
\multirow{4}{*}{\shortstack{SD v1.5 (0.9B)\\(w/ DDPM 100 steps)}}
& Vanilla & 7.287 & \textbf{0.440} & \textbf{0.281} & \textbf{0.520} \\
\cmidrule(l){2-6}
& SMC & \textbf{7.418} & 0.439 & 0.280 & \textbf{0.520} \\
& \cellcolor{gray!25}BoNG & \cellcolor{gray!25}7.306 & \cellcolor{gray!25}0.399 & \cellcolor{gray!25}0.279 & \cellcolor{gray!25}0.512 \\
\bottomrule
\end{tabular}
\end{table}

\subsection{Composing BoNG with SMC}
\label{app:smc_bong}
Inference-time steering methods fall into two distinct classes~\citep{singhal2025a}: reward-guidance methods, which steer particles by modifying the sampling dynamics, and particle-based resampling methods, which steer by reweighting and resampling intermediate particles with potentials. BoNG shares a structural connection with SMC, as both maintain $N$ interacting particles and evaluate intermediate rewards on denoised estimates, but the two belong to different classes. BoNG is a reward-guidance method, in which the reward signal enters through the guidance term $\frac{s}{\sigma_t^2}\bigl(\hat{\rvx}_0^{\mathrm{BoN}}-\hat{\rvx}_0^i\bigr)$ added to the score function, whereas SMC is a potential-based resampling method, in which the reward signal enters through the resampling weights $G_t$ while the transition kernel is left unchanged. Concretely, while the outer loop of BoNG resembles that of SMC (propagate $N$ particles, score their denoised estimates with the reward), SMC applies multinomial resampling based on $G_t$, so particles are copied and deleted, whereas BoNG applies no resampling and instead modifies the score function. In this sense, BoNG brings the strength of the guidance class, with no particle copying or deletion, to the black-box reward setting.

This distinction has a direct practical consequence: since guidance and resampling act on different components of the particle system (the transition kernel and the resampling weights, respectively), they are orthogonal and composable. To verify this empirically, we apply BoNG on top of SMC and report ImageReward with $N$ particles in \cref{tab:smc_bong}, for both the mean over the $N$ returned samples and the Best-of-$N$ sample. Composing BoNG with SMC improves both the mean and the Best-of-$N$ reward at every particle count, with no reward gradients and no additional score-network evaluations. This empirically confirms that BoNG occupies the proposal-generator slot of the FK-IPS framework~\citep{singhal2025a} and complements resampling-based steering.

\begin{table}[H]
\small
\centering
\setlength{\tabcolsep}{6pt}
\caption{Composing BoNG with SMC. \textbf{Bold} entries indicate the best results.}
\label{tab:smc_bong}
\vspace{1.5mm}
\begin{tabular}{clcccc}
\toprule
& & \multicolumn{2}{c}{\textbf{Mean IR} $(\uparrow)$} & \multicolumn{2}{c}{\textbf{BoN IR} $(\uparrow)$} \\
\cmidrule(lr){3-4}\cmidrule(l){5-6}
\textbf{Backbone} & \textbf{Particle Count} & SMC & +BoNG & SMC & +BoNG \\
\midrule
\multirow{3}{*}{\shortstack{SD v1.5 (0.9B)\\(w/ DDPM 100 steps)}}
& $N=2$ & $-0.039$ & \textbf{0.276} & 0.377 & \textbf{0.417} \\
& $N=3$ & 0.463 & \textbf{0.549} & 0.679 & \textbf{0.684} \\
& $N=4$ & 0.641 & \textbf{0.706} & 0.823 & \textbf{0.833} \\
\bottomrule
\end{tabular}
\end{table}

\section{Qualitative Comparison}
\label{app:qualitative}
\cref{fig:bon_samples_compare} shows the Best-of-$N$ samples of Vanilla, SMC, and BoNG for various prompts under the configurations of \cref{Table1}. BoNG produces samples that are better aligned with the prompts across backbones and samplers.

\begin{figure}[H]
    \centering
    \includegraphics[width=\textwidth]{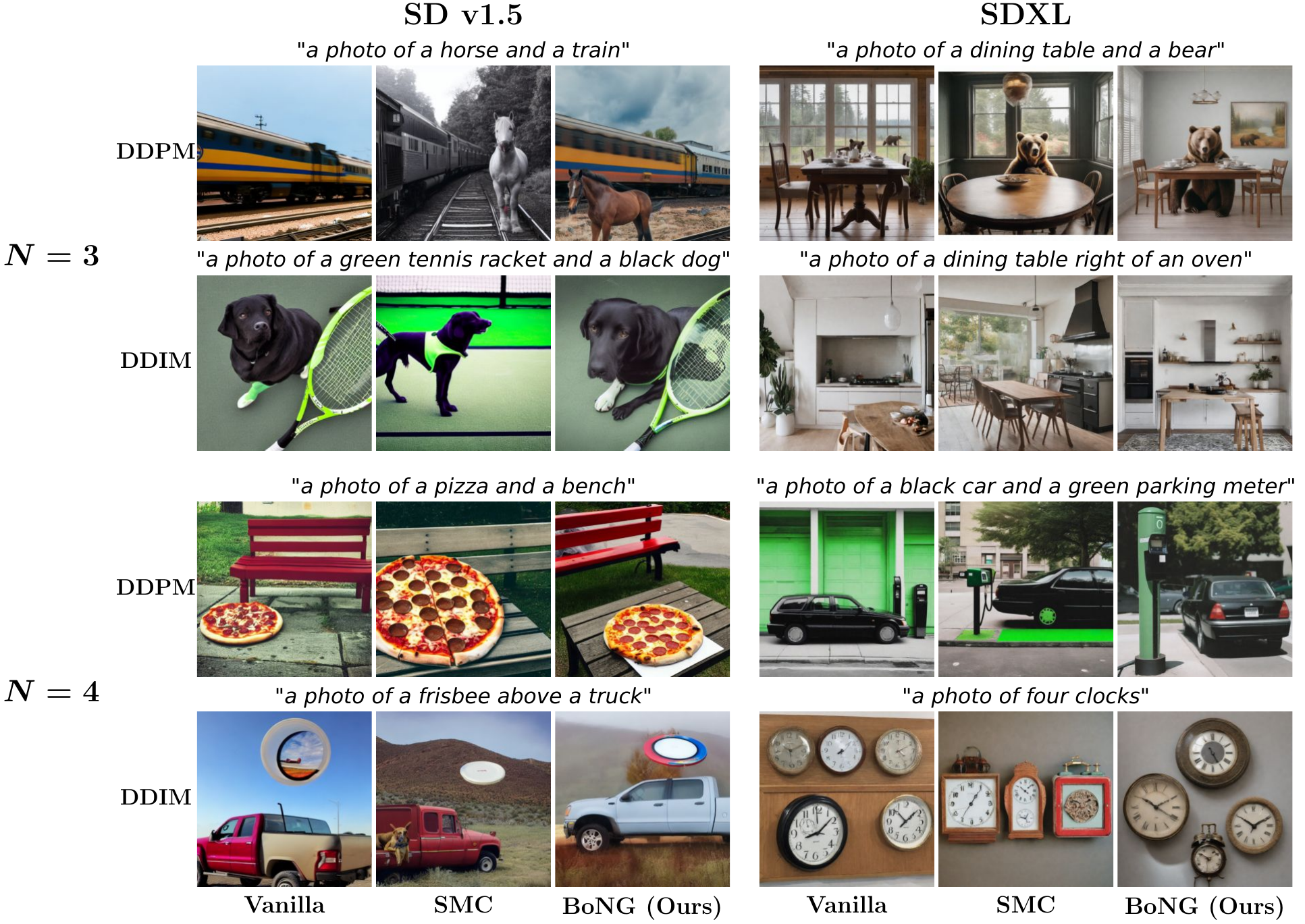}
    \vspace{-2mm}
    \caption{Qualitative comparison of Best-of-$N$ samples from Vanilla, SMC, and BoNG.}
    \label{fig:bon_samples_compare}
\end{figure}

\section{Pseudocode of BoNG}
\label{app:pseudocode}

\cref{alg:bong_full} presents the full sampling procedure of BoNG, combining the Best-of-$N$ selection of \cref{alg:bon-select} and the guided update of \cref{alg:bong-update} into a single loop over the reverse diffusion process.

\begin{algorithm}[H]
\caption{BoNG: full sampling procedure}
\label{alg:bong_full}
\begin{algorithmic}[1]
\STATE \textbf{Input:} prompt $\rvc$, number of particles $N$, diffusion steps $T$, guidance scale $s$, guidance window $[t_{\mathrm{start}}, t_{\mathrm{end}}]$, reward $r$, pretrained noise network $\epsilon_\theta$
\STATE Initialize $\rvx_T^i \sim \mathcal{N}(\mathbf{0}, \mathbf{I})$ for $i=1,\dots,N$
\FOR{$t = T, \dots, 1$}
    \STATE $\bm{\epsilon}^i \leftarrow \epsilon_\theta(\rvx_t^i, t, \rvc)$ for $i=1,\dots,N$
    \IF{$t_{\mathrm{start}} \le t < t_{\mathrm{end}}$}
        \STATE $\hat{\rvx}_0^i \leftarrow \bigl(\rvx_t^i - \sigma_t \bm{\epsilon}^i\bigr)/\sqrt{\bar\alpha_t}$ for $i=1,\dots,N$ \hfill $\triangleright$ Alg.~\ref{alg:bon-select}
        \STATE $R^i \leftarrow r\bigl(\mathrm{Decode}(\hat{\rvx}_0^i), \rvc\bigr)$ for $i=1,\dots,N$
        \STATE $i^* \leftarrow \arg\max_i R^i$, \quad $\hat{\rvx}_0^{\mathrm{BoN}} \leftarrow \hat{\rvx}_0^{i^*}$
        \STATE $\bm{\epsilon}^i \leftarrow \bm{\epsilon}^i - \sqrt{\bar\alpha_t}\, s \,\bigl(\hat{\rvx}_0^{\mathrm{BoN}} - \hat{\rvx}_0^i\bigr)/\sigma_t$ for $i=1,\dots,N$ \hfill $\triangleright$ Alg.~\ref{alg:bong-update}
    \ENDIF
    \STATE $\rvx_{t-1}^i \leftarrow \mathrm{SchedulerStep}(\bm{\epsilon}^i, \rvx_t^i, t)$ for $i=1,\dots,N$
\ENDFOR
\STATE \textbf{Output:} $\{\rvx_0^i\}_{i=1}^N$
\end{algorithmic}
\end{algorithm}

\newpage

\input{checklist.tex}

\end{document}

%% file: math_commands.tex
\usepackage{amsmath,amsfonts,bm}

\def\eqref#1{equation~\ref{#1}}

\def\1{\bm{1}}

\def\rvc{{\mathbf{c}}}

\def\rvx{{\mathbf{x}}}

\DeclareMathAlphabet{\mathsfit}{\encodingdefault}{\sfdefault}{m}{sl}
\SetMathAlphabet{\mathsfit}{bold}{\encodingdefault}{\sfdefault}{bx}{n}



%% file: checklist.tex
\section*{NeurIPS Paper Checklist}

\begin{enumerate}

\item {\bf Claims}
    \item[] Question: Do the main claims made in the abstract and introduction accurately reflect the paper's contributions and scope?
    \item[] Answer: \answerYes{}{} 
    \item[] Justification:  Both part accurately reflect the paper's contributions and scope.
    \item[] Guidelines: 
    \begin{itemize}
        \item The answer \answerNA{} means that the abstract and introduction do not include the claims made in the paper.
        \item The abstract and/or introduction should clearly state the claims made, including the contributions made in the paper and important assumptions and limitations. A \answerNo{} or \answerNA{} answer to this question will not be perceived well by the reviewers. 
        \item The claims made should match theoretical and experimental results, and reflect how much the results can be expected to generalize to other settings. 
        \item It is fine to include aspirational goals as motivation as long as it is clear that these goals are not attained by the paper. 
    \end{itemize}

\item {\bf Limitations}
    \item[] Question: Does the paper discuss the limitations of the work performed by the authors?
    \item[] Answer: \answerYes{}{} 
    \item[] Justification: The limitations are discussed in the main text.
    \item[] Guidelines:
    \begin{itemize}
        \item The answer \answerNA{} means that the paper has no limitation while the answer \answerNo{} means that the paper has limitations, but those are not discussed in the paper. 
        \item The authors are encouraged to create a separate ``Limitations'' section in their paper.
        \item The paper should point out any strong assumptions and how robust the results are to violations of these assumptions (e.g., independence assumptions, noiseless settings, model well-specification, asymptotic approximations only holding locally). The authors should reflect on how these assumptions might be violated in practice and what the implications would be.
        \item The authors should reflect on the scope of the claims made, e.g., if the approach was only tested on a few datasets or with a few runs. In general, empirical results often depend on implicit assumptions, which should be articulated.
        \item The authors should reflect on the factors that influence the performance of the approach. For example, a facial recognition algorithm may perform poorly when image resolution is low or images are taken in low lighting. Or a speech-to-text system might not be used reliably to provide closed captions for online lectures because it fails to handle technical jargon.
        \item The authors should discuss the computational efficiency of the proposed algorithms and how they scale with dataset size.
        \item If applicable, the authors should discuss possible limitations of their approach to address problems of privacy and fairness.
        \item While the authors might fear that complete honesty about limitations might be used by reviewers as grounds for rejection, a worse outcome might be that reviewers discover limitations that aren't acknowledged in the paper. The authors should use their best judgment and recognize that individual actions in favor of transparency play an important role in developing norms that preserve the integrity of the community. Reviewers will be specifically instructed to not penalize honesty concerning limitations.
    \end{itemize}

\item {\bf Theory assumptions and proofs}
    \item[] Question: For each theoretical result, does the paper provide the full set of assumptions and a complete (and correct) proof?
    \item[] Answer: \answerYes{}{} 
    \item[] Justification: We provide a full derivation and proof for our Lemma, and Theorem in \cref{app:lem}\, \cref{app:lemma1}, and \cref{app:theorem} respectively.
    \item[] Guidelines:
    \begin{itemize}
        \item The answer \answerNA{} means that the paper does not include theoretical results. 
        \item All the theorems, formulas, and proofs in the paper should be numbered and cross-referenced.
        \item All assumptions should be clearly stated or referenced in the statement of any theorems.
        \item The proofs can either appear in the main paper or the supplemental material, but if they appear in the supplemental material, the authors are encouraged to provide a short proof sketch to provide intuition. 
        \item Inversely, any informal proof provided in the core of the paper should be complemented by formal proofs provided in appendix or supplemental material.
        \item Theorems and Lemmas that the proof relies upon should be properly referenced. 
    \end{itemize}

    \item {\bf Experimental result reproducibility}
    \item[] Question: Does the paper fully disclose all the information needed to reproduce the main experimental results of the paper to the extent that it affects the main claims and/or conclusions of the paper (regardless of whether the code and data are provided or not)?
    \item[] Answer: \answerYes{}{} 
    \item[] Justification: We state the configuration in \cref{Sec:B.2}
    \item[] Guidelines:
    \begin{itemize}
        \item The answer \answerNA{} means that the paper does not include experiments.
        \item If the paper includes experiments, a \answerNo{} answer to this question will not be perceived well by the reviewers: Making the paper reproducible is important, regardless of whether the code and data are provided or not.
        \item If the contribution is a dataset and\slash or model, the authors should describe the steps taken to make their results reproducible or verifiable. 
        \item Depending on the contribution, reproducibility can be accomplished in various ways. For example, if the contribution is a novel architecture, describing the architecture fully might suffice, or if the contribution is a specific model and empirical evaluation, it may be necessary to either make it possible for others to replicate the model with the same dataset, or provide access to the model. In general. releasing code and data is often one good way to accomplish this, but reproducibility can also be provided via detailed instructions for how to replicate the results, access to a hosted model (e.g., in the case of a large language model), releasing of a model checkpoint, or other means that are appropriate to the research performed.
        \item While NeurIPS does not require releasing code, the conference does require all submissions to provide some reasonable avenue for reproducibility, which may depend on the nature of the contribution. For example
        \begin{enumerate}
            \item If the contribution is primarily a new algorithm, the paper should make it clear how to reproduce that algorithm.
            \item If the contribution is primarily a new model architecture, the paper should describe the architecture clearly and fully.
            \item If the contribution is a new model (e.g., a large language model), then there should either be a way to access this model for reproducing the results or a way to reproduce the model (e.g., with an open-source dataset or instructions for how to construct the dataset).
            \item We recognize that reproducibility may be tricky in some cases, in which case authors are welcome to describe the particular way they provide for reproducibility. In the case of closed-source models, it may be that access to the model is limited in some way (e.g., to registered users), but it should be possible for other researchers to have some path to reproducing or verifying the results.
        \end{enumerate}
    \end{itemize}

\item {\bf Open access to data and code}
    \item[] Question: Does the paper provide open access to the data and code, with sufficient instructions to faithfully reproduce the main experimental results, as described in supplemental material?
    \item[] Answer: \answerYes{} 
    \item[] Justification: We provide the baseline configuration settings and additional details thoroughly in \cref{sec:B.1,Sec:B.2}.
    \item[] Guidelines:
    \begin{itemize}
        \item The answer \answerNA{} means that paper does not include experiments requiring code.
        \item Please see the NeurIPS code and data submission guidelines (\url{https://neurips.cc/public/guides/CodeSubmissionPolicy}) for more details.
        \item While we encourage the release of code and data, we understand that this might not be possible, so \answerNo{} is an acceptable answer. Papers cannot be rejected simply for not including code, unless this is central to the contribution (e.g., for a new open-source benchmark).
        \item The instructions should contain the exact command and environment needed to run to reproduce the results. See the NeurIPS code and data submission guidelines (\url{https://neurips.cc/public/guides/CodeSubmissionPolicy}) for more details.
        \item The authors should provide instructions on data access and preparation, including how to access the raw data, preprocessed data, intermediate data, and generated data, etc.
        \item The authors should provide scripts to reproduce all experimental results for the new proposed method and baselines. If only a subset of experiments are reproducible, they should state which ones are omitted from the script and why.
        \item At submission time, to preserve anonymity, the authors should release anonymized versions (if applicable).
        \item Providing as much information as possible in supplemental material (appended to the paper) is recommended, but including URLs to data and code is permitted.
    \end{itemize}

\item {\bf Experimental setting/details}
    \item[] Question: Does the paper specify all the training and test details (e.g., data splits, hyperparameters, how they were chosen, type of optimizer) necessary to understand the results?
    \item[] Answer: \answerYes{}{} 
    \item[] Justification: We specifically state the configuration for all experiments done in our paper in \cref{Sec:B.2}. 
    \item[] Guidelines:
    \begin{itemize}
        \item The answer \answerNA{} means that the paper does not include experiments.
        \item The experimental setting should be presented in the core of the paper to a level of detail that is necessary to appreciate the results and make sense of them.
        \item The full details can be provided either with the code, in appendix, or as supplemental material.
    \end{itemize}

\item {\bf Experiment statistical significance}
    \item[] Question: Does the paper report error bars suitably and correctly defined or other appropriate information about the statistical significance of the experiments?
    \item[] Answer: \answerNo{} 
    \item[] Justification: We do not provided error bars or statistical significance but instead support our main claim by various experiments on various backbone models and also by analyzing the trend and tradeoff.
    \item[] Guidelines:
    \begin{itemize}
        \item The answer \answerNA{} means that the paper does not include experiments.
        \item The authors should answer \answerYes{} if the results are accompanied by error bars, confidence intervals, or statistical significance tests, at least for the experiments that support the main claims of the paper.
        \item The factors of variability that the error bars are capturing should be clearly stated (for example, train/test split, initialization, random drawing of some parameter, or overall run with given experimental conditions).
        \item The method for calculating the error bars should be explained (closed form formula, call to a library function, bootstrap, etc.)
        \item The assumptions made should be given (e.g., Normally distributed errors).
        \item It should be clear whether the error bar is the standard deviation or the standard error of the mean.
        \item It is OK to report 1-sigma error bars, but one should state it. The authors should preferably report a 2-sigma error bar than state that they have a 96\% CI, if the hypothesis of Normality of errors is not verified.
        \item For asymmetric distributions, the authors should be careful not to show in tables or figures symmetric error bars that would yield results that are out of range (e.g., negative error rates).
        \item If error bars are reported in tables or plots, the authors should explain in the text how they were calculated and reference the corresponding figures or tables in the text.
    \end{itemize}

\item {\bf Experiments compute resources}
    \item[] Question: For each experiment, does the paper provide sufficient information on the computer resources (type of compute workers, memory, time of execution) needed to reproduce the experiments?
    \item[] Answer: \answerYes{}{} 
    \item[] Justification: We report our experiments was done under a specific GPU in \cref{Table1,Table2}.
    \item[] Guidelines:
    \begin{itemize}
        \item The answer \answerNA{} means that the paper does not include experiments.
        \item The paper should indicate the type of compute workers CPU or GPU, internal cluster, or cloud provider, including relevant memory and storage.
        \item The paper should provide the amount of compute required for each of the individual experimental runs as well as estimate the total compute. 
        \item The paper should disclose whether the full research project required more compute than the experiments reported in the paper (e.g., preliminary or failed experiments that didn't make it into the paper). 
    \end{itemize}
    
\item {\bf Code of ethics}
    \item[] Question: Does the research conducted in the paper conform, in every respect, with the NeurIPS Code of Ethics \url{https://neurips.cc/public/EthicsGuidelines}?
    \item[] Answer: \answerYes{}{} 
    \item[] Justification: We keep the code of ethics.
    \item[] Guidelines:
    \begin{itemize}
        \item The answer \answerNA{} means that the authors have not reviewed the NeurIPS Code of Ethics.
        \item If the authors answer \answerNo, they should explain the special circumstances that require a deviation from the Code of Ethics.
        \item The authors should make sure to preserve anonymity (e.g., if there is a special consideration due to laws or regulations in their jurisdiction).
    \end{itemize}

\item {\bf Broader impacts}
    \item[] Question: Does the paper discuss both potential positive societal impacts and negative societal impacts of the work performed?
    \item[] Answer: \answerYes{}{} 
    \item[] Justification: We discuss it in the main text.
    \item[] Guidelines:
    \begin{itemize}
        \item The answer \answerNA{} means that there is no societal impact of the work performed.
        \item If the authors answer \answerNA{} or \answerNo, they should explain why their work has no societal impact or why the paper does not address societal impact.
        \item Examples of negative societal impacts include potential malicious or unintended uses (e.g., disinformation, generating fake profiles, surveillance), fairness considerations (e.g., deployment of technologies that could make decisions that unfairly impact specific groups), privacy considerations, and security considerations.
        \item The conference expects that many papers will be foundational research and not tied to particular applications, let alone deployments. However, if there is a direct path to any negative applications, the authors should point it out. For example, it is legitimate to point out that an improvement in the quality of generative models could be used to generate Deepfakes for disinformation. On the other hand, it is not needed to point out that a generic algorithm for optimizing neural networks could enable people to train models that generate Deepfakes faster.
        \item The authors should consider possible harms that could arise when the technology is being used as intended and functioning correctly, harms that could arise when the technology is being used as intended but gives incorrect results, and harms following from (intentional or unintentional) misuse of the technology.
        \item If there are negative societal impacts, the authors could also discuss possible mitigation strategies (e.g., gated release of models, providing defenses in addition to attacks, mechanisms for monitoring misuse, mechanisms to monitor how a system learns from feedback over time, improving the efficiency and accessibility of ML).
    \end{itemize}
    
\item {\bf Safeguards}
    \item[] Question: Does the paper describe safeguards that have been put in place for responsible release of data or models that have a high risk for misuse (e.g., pre-trained language models, image generators, or scraped datasets)?
    \item[] Answer: \answerNo{}{} 
    \item[] Justification: This aspect is not typically regarded within the scope of prior reward guidance methods. 
    \item[] Guidelines:
    \begin{itemize}
        \item The answer \answerNA{} means that the paper poses no such risks.
        \item Released models that have a high risk for misuse or dual-use should be released with necessary safeguards to allow for controlled use of the model, for example by requiring that users adhere to usage guidelines or restrictions to access the model or implementing safety filters. 
        \item Datasets that have been scraped from the Internet could pose safety risks. The authors should describe how they avoided releasing unsafe images.
        \item We recognize that providing effective safeguards is challenging, and many papers do not require this, but we encourage authors to take this into account and make a best faith effort.
    \end{itemize}

\item {\bf Licenses for existing assets}
    \item[] Question: Are the creators or original owners of assets (e.g., code, data, models), used in the paper, properly credited and are the license and terms of use explicitly mentioned and properly respected?
    \item[] Answer: \answerYes{}{} 
    \item[] Justification: For all comparison to baseline methods we accordingly cite the paper in the reference section.  
    \item[] Guidelines:
    \begin{itemize}
        \item The answer \answerNA{} means that the paper does not use existing assets.
        \item The authors should cite the original paper that produced the code package or dataset.
        \item The authors should state which version of the asset is used and, if possible, include a URL.
        \item The name of the license (e.g., CC-BY 4.0) should be included for each asset.
        \item For scraped data from a particular source (e.g., website), the copyright and terms of service of that source should be provided.
        \item If assets are released, the license, copyright information, and terms of use in the package should be provided. For popular datasets, \url{paperswithcode.com/datasets} has curated licenses for some datasets. Their licensing guide can help determine the license of a dataset.
        \item For existing datasets that are re-packaged, both the original license and the license of the derived asset (if it has changed) should be provided.
        \item If this information is not available online, the authors are encouraged to reach out to the asset's creators.
    \end{itemize}

\item {\bf New assets}
    \item[] Question: Are new assets introduced in the paper well documented and is the documentation provided alongside the assets?
    \item[] Answer: \answerNA{} 
    \item[] Justification: Our paper does not introduce a new evaluation dataset. We use the standard evaluation dataset, GenEval \cite{ghosh2023geneval} prompt.
    \item[] Guidelines:
    \begin{itemize}
        \item The answer \answerNA{} means that the paper does not release new assets.
        \item Researchers should communicate the details of the dataset\slash code\slash model as part of their submissions via structured templates. This includes details about training, license, limitations, etc. 
        \item The paper should discuss whether and how consent was obtained from people whose asset is used.
        \item At submission time, remember to anonymize your assets (if applicable). You can either create an anonymized URL or include an anonymized zip file.
    \end{itemize}

\item {\bf Crowdsourcing and research with human subjects}
    \item[] Question: For crowdsourcing experiments and research with human subjects, does the paper include the full text of instructions given to participants and screenshots, if applicable, as well as details about compensation (if any)? 
    \item[] Answer: \answerNA{}{} 
    \item[] Justification: Our paper does not do crowdsourcing and research with human subjects.
    \item[] Guidelines:
    \begin{itemize}
        \item The answer \answerNA{} means that the paper does not involve crowdsourcing nor research with human subjects.
        \item Including this information in the supplemental material is fine, but if the main contribution of the paper involves human subjects, then as much detail as possible should be included in the main paper. 
        \item According to the NeurIPS Code of Ethics, workers involved in data collection, curation, or other labor should be paid at least the minimum wage in the country of the data collector. 
    \end{itemize}

\item {\bf Institutional review board (IRB) approvals or equivalent for research with human subjects}
    \item[] Question: Does the paper describe potential risks incurred by study participants, whether such risks were disclosed to the subjects, and whether Institutional Review Board (IRB) approvals (or an equivalent approval/review based on the requirements of your country or institution) were obtained?
    \item[] Answer: \answerNA{}{} 
    \item[] Justification: Our paper does not contain crowdsourcing.
    \item[] Guidelines:
    \begin{itemize}
        \item The answer \answerNA{} means that the paper does not involve crowdsourcing nor research with human subjects.
        \item Depending on the country in which research is conducted, IRB approval (or equivalent) may be required for any human subjects research. If you obtained IRB approval, you should clearly state this in the paper. 
        \item We recognize that the procedures for this may vary significantly between institutions and locations, and we expect authors to adhere to the NeurIPS Code of Ethics and the guidelines for their institution. 
        \item For initial submissions, do not include any information that would break anonymity (if applicable), such as the institution conducting the review.
    \end{itemize}

\item {\bf Declaration of LLM usage}
    \item[] Question: Does the paper describe the usage of LLMs if it is an important, original, or non-standard component of the core methods in this research? Note that if the LLM is used only for writing, editing, or formatting purposes and does \emph{not} impact the core methodology, scientific rigor, or originality of the research, declaration is not required.
    \item[] Answer: \answerNA{} 
    \item[] Justification: We do not use LLM evaluation as a metric for performance evaluations. 
    \item[] Guidelines:
    \begin{itemize}
        \item The answer \answerNA{} means that the core method development in this research does not involve LLMs as any important, original, or non-standard components.
        \item Please refer to our LLM policy in the NeurIPS handbook for what should or should not be described.
    \end{itemize}

\end{enumerate}